\documentclass{article}
\usepackage{microtype}
\usepackage{graphicx}
\usepackage{subfigure}
\usepackage{booktabs} 
\usepackage{dsfont}
\usepackage{amssymb}
\usepackage{amsmath}
\usepackage{graphicx}
\usepackage{epstopdf}
\usepackage{url}
\usepackage{appendix}
\usepackage{amsthm}
\usepackage{thmtools}
\usepackage{thm-restate}
\usepackage{float}
\RequirePackage[loading]{tracefnt}
\usepackage{hyperref}

\newcommand{\indicator}[1]{\mathds{1}{#1}}

\DeclareMathOperator*{\argmax}{argmax}
\DeclareMathOperator*{\argmin}{argmin}
\DeclareMathOperator{\mE}{\mathbb{E}}
\DeclareMathOperator{\mV}{\mathbb{V}}
\DeclareMathOperator{\mX}{\mathcal{X}}
\DeclareMathOperator{\mY}{\mathcal{Y}}

\DeclareMathOperator{\mCg}{\indicator{\{\hat{c}> M\}}}
\DeclareMathOperator{\mCs}{\indicator{\{\hat{c} \leq M\}}}
\DeclareMathOperator{\wa}{w^{\alpha}_{i\bar{y}}}
\DeclareMathOperator{\wb}{w^{\beta}_{i}}
\DeclareMathOperator{\wg}{w^{\gamma}_{i}}

\DeclareMathOperator{\wxa}{w^{\alpha}_{x\bar{y}}}
\DeclareMathOperator{\wxya}{w^{\alpha}_{xy}}
\DeclareMathOperator{\wxb}{w^{\beta}_{xy}}
\DeclareMathOperator{\wxg}{w^{\gamma}_{xy}}

\DeclareMathOperator{\wva}{w^{\alpha}}
\DeclareMathOperator{\wvb}{w^{\beta}}
\DeclareMathOperator{\wvg}{w^{\gamma}}

\DeclareMathOperator{\rp}{R^{\mathbf{w}}_{xy}(\pi)}
\DeclareMathOperator{\Rw}{\hat{R}^{\mathbf{w}}}
\DeclareMathOperator{\Rwi}{\hat{R}^{\mathbf{w}}_i}

\newtheorem*{definition}{Definition}
\newtheorem{condition}{Condition}

\newcommand{\blendnameNS}{Continuous Adaptive Blending}
\newcommand{\classname}{Interpolated Counterfactual Estimator Family}
\newcommand{\blendname}{\blendnameNS\ }
\newcommand{\blendabbrNS}{CAB}
\newcommand{\blendabbr}{\blendabbrNS\ }
\newcommand{\newy}{\bar{y}}
\newcommand{\Rcab}{\hat{R}_{\blendabbrNS}}
\newcommand{\Rcabdr}{\hat{R}_{CABDR}}

\usepackage[accepted]{icml2019}

\icmltitlerunning{CAB: Continuous Adaptive Blending for Policy Evaluation and Learning}

\begin{document}

\twocolumn[
\icmltitle{CAB: Continuous Adaptive Blending for Policy Evaluation and Learning}



\icmlsetsymbol{equal}{*}

\begin{icmlauthorlist}
\icmlauthor{Yi Su}{equal,cornell}
\icmlauthor{Lequn Wang}{equal,cornell}
\icmlauthor{Michele Santacatterina}{tripods}
\icmlauthor{Thorsten Joachims}{cornell}
\end{icmlauthorlist}

\icmlaffiliation{cornell}{Cornell University, Ithaca, USA}
\icmlaffiliation{tripods}{Cornell TRIPODS Center for Data Science, Ithaca, USA}

\icmlcorrespondingauthor{Yi Su}{ys756@cornell.edu}
\icmlcorrespondingauthor{Lequn Wang}{lw633@cornell.edu}
\icmlcorrespondingauthor{Michele Santacatterina}{santacatterina@cornell.edu}
\icmlcorrespondingauthor{Thorsten Joachims}{tj@cs.cornell.edu}
\icmlkeywords{Information Retrieval, Recommender System, Causality, Ranking and Preference Learning}

\vskip 0.3in
]


\printAffiliationsAndNotice{\icmlEqualContribution} 

\begin{abstract}
The ability to perform offline A/B-testing and off-policy learning using logged contextual bandit feedback is highly desirable in a broad range of applications, including recommender systems, search engines, ad placement, and personalized health care. Both offline A/B-testing and off-policy learning require a counterfactual estimator that evaluates how some new policy would have performed, if it had been used instead of the logging policy. 
In this paper, we present and analyze a family of counterfactual estimators which subsumes most estimators proposed to date. Most importantly, this analysis identifies a new estimator -- called \blendname (\blendabbrNS) -- which enjoys many advantageous theoretical and practical properties. In particular, it can be substantially less biased than clipped Inverse Propensity Score (IPS) weighting and the Direct Method, and it can have less variance than Doubly Robust and IPS estimators. In addition, it is sub-differentiable such that it can be used for learning, unlike the SWITCH estimator. Experimental results show that \blendabbr provides excellent evaluation accuracy and outperforms other counterfactual estimators in terms of learning performance.
\end{abstract}

\section{Introduction}
\label{introduction}
Contextual bandit feedback is ubiquitous in a wide range of intelligent systems that interact with their users through the following process. The system observes a context, takes an action, and then observes feedback under the chosen action. The logs of search engines, recommender systems, ad-placement systems, and many other systems contain terabytes of this partial-information feedback data, and it is highly desirable to use this historic data for offline evaluation and learning. 
What makes evaluation and learning challenging, however, is that we only get to see the feedback for the chosen action, but we do not observe how the user would have responded if the system had taken a different action. This means the feedback is both partial (e.g. only observed for the selected action) and biased (e.g. by the choices of the policy that logged the data), which makes batch learning from contextual bandit feedback substantially different from typical supervised learning, where the correct label and a loss function provide full-information feedback.  

Both learning and evaluation can be viewed as examples of \textit{counterfactual reasoning}, where we need to estimate from the historic log data how well some other policy would have performed, if we had used it instead of the policy that logged the data. Three main approaches have been proposed for this counterfactual or off-policy evaluation problem. First, the Direct Method (DM) \cite{schafer1997analysis, rubin2004multiple, dudik2011doubly, little2019statistical} uses regression to learn a model of the reward and imputes the missing feedback. Second, inverse propensity score (IPS) weighting \cite{horvitz1952generalization, strehl2010learning} models the selection bias in the assignment mechanism and directly provides an unbiased estimate of the quality of a policy under suitable common support conditions. Both approaches are complementary and have different strengths and drawbacks. On the one hand, DM typically has low variance but can lead to highly biased results due to model misspecification. On the other hand, since we are typically controlling the logging policy and can log propensities, IPS-based methods can be provably unbiased but often suffer from large variance when the IPS weights are large. The third class is a hybrid of them. The most prominent one is the Doubly Robust (DR) estimator \cite{robins1995semiparametric, kang2007demystifying, dudik2011doubly}, which is based on DM but also uses IPS weighting and an additive control variate to reduce variance.

Generalizing these existing counterfactual estimators, we present a parametric family of estimators for off-policy evaluation that covers most of the off-policy estimators proposed to date --- including IPS \cite{horvitz1952generalization}, clipped IPS \cite{strehl2010learning}, DM \cite{dudik2011doubly}, DR \cite{dudik2011doubly} and its MRDR variant \cite{fara2018MRDR}, SWITCH \cite{wang2016optimal}, and Static Blending (SB) \cite{thomas2016data}. Providing a general bias-variance analysis for this family of estimators, we find that there is a particular new estimator in this family that has many desirable properties. We call this new estimator \blendname (\blendabbrNS) and show how it blends the complementary strengths of DM and IPS, thus providing an effective tool for optimizing the bias of DM against the variance of IPS. Compared to existing estimators, we find that \blendabbr can be substantially less biased than clipped IPS and DM while having lower variance compared to IPS and DR estimators. Furthermore, compared to estimators that perform static blending (SB) \cite{thomas2016data}, \blendabbr is adaptive to the IPS weights and handles violations of the support condition gracefully. Unlike SWITCH \cite{wang2016optimal}, \blendabbr is sub-differentiable which allows its use as the training objective in off-policy learning algorithms like POEM \cite{Swaminathan/Joachims/15c} and BanditNet \cite{Joachims/etal/18a}. Finally, unlike the DR estimator, \blendabbr can be used in off-policy Learning to Rank (LTR) algorithms like \cite{joachims2017unbiased}, and \blendabbr is specifically designed to control the bias/variance trade-off. We evaluate \blendabbr both theoretically and empirically. In particular, we present theoretical results that characterize the bias and variance of \blendabbrNS. Furthermore, we provide an extensive empirical evaluation of \blendabbr on both contextual-bandit problems and partial-information ranking problems, including real-world data from Amazon Music.

\section{Off-policy Evaluation in Contextual Bandits}

Before presenting our general family of counterfactual estimators, we begin with a formal definition of off-policy evaluation and learning in the contextual-bandit setting. To keep notation and exposure simple, we focus on the contextual bandit setting and do not explicitly consider other partial-information settings like counterfactual LTR \cite{joachims2017unbiased}. However, most estimators can be translated into that setting as well, and we discuss further details in the context of the ranking experiments in Section~\ref{sec:experiments} and in Appendix~\ref{sec:ltr}.

\subsection{Contextual-Bandit Setting and Learning}

In the contextual-bandit setting, a context $x \in \mathcal{X}$ (e.g., user profile, query) is drawn i.i.d.\ from some unknown $P(\mathcal{X})$, the deployed (stochastic) policy $\pi_0(y|x)$ then selects an action $y \in \mathcal{Y}$, and the system receives feedback $r\sim D(r|x,y)$ for this particular (context, action) pair. However, we do not observe feedback for any of the other actions. 
The logged contextual bandit data we get from logging policy $\pi_0$ is of the form 
\begin{equation}
  \mathcal{S}=\{(x_i,y_i,r_i,\pi_0(\cdot|x_i))\}_{i=1}^n, \label{eq:sample}
   \vspace*{-0.2cm}
\end{equation}
where $r_i:=r(x_i, y_i)$ is the observed reward. If the logging policy $\pi_0$ is unknown, an estimate $\hat{\pi}_0$ of the logging policy is used. 
Off-policy evaluation refers to the problem of using $\mathcal{S}$ for estimating the expected reward (or loss) $R$ of a new policy $\pi$
\begin{equation}
R(\pi)= \mE_{x\sim P(x)}\mE_{\bar{y}\sim\pi(\bar{y}|x)}\mE_{r \sim D(r|x,\bar{y})}[r]. \label{eq:reward}
\end{equation}
Off-policy learning uses $\mathcal{S}$ for finding an optimal policy $\pi^* \in \Pi$ that maximizes the expected reward (or minimizes the expected loss)
\begin{equation}
    \pi^* = \argmax_{\pi \in \Pi} \big[R(\pi)\big] \label{eq:pistar}.
    \vspace*{-0.2cm}
\end{equation}
The expected reward $R(\pi)$ cannot be computed directly, and it is typically replaced with a counterfactual estimate $\hat{R}(\pi)$ that can be computed from the logged bandit feedback $\mathcal{S}$ \cite{bottou2013counterfactual}. 
This enables Empirical Risk Minimization (ERM) for batch learning from bandit feedback (BLBF) \cite{zadrozny2003cost, beygelzimer2009offset, strehl2010learning, Swaminathan/Joachims/15c, Swaminathan/Joachims/15d}, where the algorithm finds the policy in $\Pi$ that maximizes the estimated expected reward
\begin{equation}
    \hat{\pi}^* = \argmax_{\pi \in \Pi}\left[\hat{R}(\pi)\right], \label{eq:erm}
    \vspace*{-0.2cm}
\end{equation}
possibly subject to capacity and variance regularization \cite{Swaminathan/Joachims/15c}. Since the counterfactual estimator $\hat{R}(\pi)$ is at the core of ERM learning, it is expected that an improved estimator will also lead to improved learning performance \cite{strehl2010learning}. In particular, unlike in supervised learning, the counterfactual estimator can have vastly different bias and variance for different policies in $\Pi$, such that trading off bias and variance of the estimator becomes important for effective learning \cite{Swaminathan/Joachims/15d}.

\subsection{Interpolated Counterfactual Estimator Family (ICE Family)}


In this section, we present a general family of estimators for off-policy evaluation in the contextual bandit setting. An analogous family of estimators for the setting of partial-information learning-to-rank \cite{joachims2017unbiased} is described in Appendix~\ref{sec:ltr}.
 
Let $\hat{\delta}(x,\newy)$ be the estimated reward for action $\newy$ given context $x$, and let $\hat{\pi}_0$ be the estimated (or known) logging policy. Then our family of off-policy estimators is defined as follows, where $\mathbf{w}=(\wva,\wvb,\wvg)$ is a triplet of weighting functions that parameterizes the family.
\begin{definition}[\classname]
Given a triplet $\mathbf{w}=(\wva,\wvb,\wvg)$ of weighting functions, 
\begin{equation*}
\begin{split}
    \Rw(\pi) =  & \frac{1}{n}\sum_{i=1}^n \sum_{\bar{y}\in\mathcal{Y}} \pi(\bar{y}|x_i) \: \wa \: \alpha_{i\bar{y}} \\
    & + \frac{1}{n}\sum_{i=1}^n \pi(y_i|x_i) \: \wb \: \beta_i 
    + \frac{1}{n}\sum_{i=1}^n \pi(y_i|x_i) \: \wg \: \gamma_i   
\end{split}
\end{equation*}
\begin{equation*}
\begin{split}
\hspace*{-0.5cm}\mbox{with } \:\:\:\:\:\:\:\:\:\:\:& \alpha_{i\bar{y}}:= \alpha(x_i,\bar{y}) := \hat{\delta}(x_i,\bar{y}), \\
             & \beta_i := \beta(x_i,y_i) :=  \frac{r(x_i,y_i)}{\hat{\pi}_0(y_i|x_i)}, \\
             & \gamma_i:= \gamma(x_i,y_i) :=  \frac{\hat{\delta}(x_i,y_i)}{\hat{\pi}_0(y_i|x_i)}.
\end{split}
\end{equation*}
\end{definition}
\begin{table}[t]
    \centering
    \caption{Overview of how the family of estimators $\Rw(\pi)$ subsumes existing off-policy estimators.}
    \label{tab:familyweights}
    \setlength{\tabcolsep}{4pt}
    \begin{tabular}{|l|c|c|c|} \hline 
    Estimator & $w_i^{\alpha}(\bar{y})$ & $w_i^{\beta}$ & $w_i^{\gamma}$ \\ \hline
    DM    &  1 & 0 & 0 \\
    IPS   &  0 & 1 & 0 \\
    cIPS  &  0 & $\min\!\left\{\!\frac{M \hat{\pi}_0(y_i|x_i)}{\pi(y_i|x_i)}\!,\!1\!\right\}$ & 0 \\ 
    DR    &  1 & 1 & -1 \\
    SB    &  $1-\tau$ & $\tau$ & 0 \\
    SWITCH&  $\indicator{\left\{\!\frac{\pi(\newy|x_i)}{\hat{\pi}_0(\newy|x_i)}\!>\!M\! \right\}}$ & $\indicator{\left\{\!\frac{\pi(y_i|x_i)}{\hat{\pi}_0(y_i|x_i)}\!\leq\! M \!\right\}}$ & 0 \\ \hline
    \end{tabular}
\end{table}
We will see in the following that different choices of the weighting functions $\mathbf{w}=(\wva,\wvb,\wvg)$ recover a wide variety of existing estimators (see Table~\ref{tab:familyweights}), and that our bias-variance analysis of $\Rw(\pi)$ suggests new estimators with desirable properties.
This class of counterfactual estimators builds upon three basic estimators for off-policy evaluation, and we will now introduce the motivation for the construction of the (context, action) level functions $\alpha(x_i,\newy), \beta(x_i,y_i)$ and $\gamma(x_i,y_i)$.

The first component $\alpha(x_i,\newy)$ follows a ``Model the World'' approach \cite{schafer1997analysis, rubin2004multiple, dudik2011doubly}, which exploits a model $\hat{\delta}(x,\newy)$ of the reward. The estimator that purely relies on this component is the DM estimator that is a special case of $\Rw(\pi)$ with static weights $\mathbf{w}=(1,0,0)$ for all $(x_i,y_i,\newy)$ (see Table~\ref{tab:familyweights}). The reward model $\hat{\delta}(x,\newy)$ is typically learned via regression, and it serves as an estimate of $\mathbb{E}_{r}[r|x,\newy]$ in \eqref{eq:reward} to be imputed in place of the unobserved rewards. The reward model typically has low variability, so the variance of DM is typically small. However, due to often unavoidable misspecification of the reward model, DM is often statistically inconsistent and can have a large bias. 

The second component $\beta(x_i,y_i)$ follows a ``Model the Bias'' approach for the assignment mechanism, which is particularly attractive in many applications where we control the assignment mechanism by design \cite{horvitz1952generalization,Swaminathan/Joachims/15c,nikosdesign2018}. The estimator fully based on this term is the widely used inverse propensity score (IPS) weighting estimator \cite{horvitz1952generalization,strehl2010learning,bottou2013counterfactual,Swaminathan/Joachims/15c}, which is the special case of $\Rw(\pi)$ with static weights $\mathbf{w}=(0,1,0)$ for all $(x_i,y_i,\newy)$ (see Table~\ref{tab:familyweights}). If $\hat{\pi}_0(y_i|x_i)$ is known and the logging policy has sufficient support w.r.t.\ the new policy $\pi$, then the IPS estimator is unbiased (see Section~\ref{sec:theoanalysis}). However, it can have large variance if the IPS weights $\pi(y_i|x_i)/\hat{\pi}_0(y_i|x_i)$ are large. 

The third component $\gamma(x_i,y_i)$ combines the "Model the World" approach and "Model the Bias" approach by de-biasing the estimated reward term. This component is not necessarily an attractive estimator itself, but can be used as part of a control variate for variance reduction as in the DR estimator \cite{robins1995semiparametric,bang2005doubly,kang2007demystifying, dudik2011doubly}. DR treats DM as a baseline while correcting the baseline when data is available, and it is a special case of $\Rw(\pi)$ with static weights $\mathbf{w}=(1,1,-1)$ for all $(x_i,y_i,\newy)$. It is unbiased when either the reward model or the propensity model is correct. However, by maintaining unbiasedness, we will discuss in Section~\ref{sec:variance_improvement} that it can still suffer from excessive variance. Furthermore, due to the non-zero weight put on the control variate term $\gamma(x_i,y_i)$, DR cannot be used for LTR from implicit feedback, since the analog of $y_i$ is only partially observable in that setting (see Appendix~\ref{sec:ltr}).

Another hybrid estimator was proposed by \cite{thomas2016data} for off-policy evaluation in the more general setting of reinforcement learning. When we translate the key idea behind their MAGIC estimator to the contextual bandit setting, we see that it is a special case of $\Rw(\pi)$ with tunable weights $\mathbf{w} =(1-\tau,  \tau,  0)$ for all $(x_i,y_i,\newy)$, where $\tau \in [0,1]$ is a parameter. We call this estimator Static Blending (SB), since $\tau$ is static and does not depend on the importance weights. 

The SWITCH Estimator \cite{wang2016optimal}, in contrast, is more adaptive. As the name implies, it switches between $\alpha(x_i,\newy)$ and $\beta(x_i,y_i)$ depending on a hard threshold $M$ on the IPS weights.
It is a special case of $\Rw(\pi)$ using the weights given in Table~\ref{tab:familyweights}.
A drawback of SWITCH is that its hard switching makes it discontinuous with respect to any parameters of the target policy $\pi$. This not only creates more erratic behavior when the threshold $M$ is changed, but it also means that SWITCH is not differentiable and thus cannot be used in gradient-based learning algorithms like POEM \cite{Swaminathan/Joachims/15c} or BanditNet \cite{Joachims/etal/18a} for BLBF. 

\subsection{Theoretical Analysis} \label{sec:theoanalysis}

We will now provide a general characterization of the bias and variance for the \classname, where both the propensity model and the reward model may be misspecified. Following \citet{dudik2011doubly}, let $\zeta(x,y)$ denote the multiplicative deviation of the propensity estimates from the true propensity model, and $\Delta(x,y)$ be the additive deviation of the reward estimates from the true expected reward.
\begin{equation}
\zeta:=\zeta(x,y) = 1 -\frac{\pi_0(y|x)}{\hat{\pi}_0(y|x)}
\end{equation}
\begin{equation}
\Delta:=\Delta(x,y) = \hat{\delta}(x,y)-\delta(x,y).
\end{equation}
Moreover, let $\sigma^2_r(x,y)$ denote the randomness of the reward. 
\begin{equation}
\sigma^2_r:=\sigma^2_r(x,y) = \mV_r(r(x,y)|x,y)
\end{equation}
Note that $\zeta(x,y)$ is zero when the logging policy $\pi_0$ is known. For brevity, we denote the true IPS weight as $c(x,y) := \frac{\pi(y|x)}{{\pi_0}(y|x)}$ and its estimated version as $\hat{c}(x,y):=\frac{\pi(y|x)}{{\hat{\pi}_0}(y|x)}$. For known propensities, $\hat{c}(x,y) = c(x,y)$. As usual, we posit that the following support/positivity condition holds.
\newline

\begin{condition}[Common Support] \label{as:commonsupport}
The logging policy $\pi_0$ has full support for the target policy $\pi$, which means  $\pi(y|x)>0 \to \pi_0(y|x)>0$ for all $x$ and $y$.
\end{condition}

\begin{restatable}[Bias of the \classname]{thm}{unibias}
\label{theo:biasunified}
For contexts $x_1, x_2, \cdots, x_n$ drawn i.i.d from some distribution $P(\mathcal{X})$ and for actions $y_i\sim \pi_0(\mathcal{Y}|x_i)$, under Condition~\ref{as:commonsupport} the bias of $\Rw(\pi)$ with weighting functions $\mathbf{w}=(\wva,\wvg,\wvg)$ is
\begin{equation}
\begin{split}
\mE_x\mE_{y\sim\pi}\Big[& \wva\Delta - \wvb\zeta\delta + \wvg(\Delta-\zeta(\delta+\Delta))\\
& +(\wva+\wvb+\wvg)\delta - \delta\Big]
\end{split}
\end{equation}
\end{restatable}
\textit{Proof:} See Appendix~\ref{Appendix: prooftheo1}.

\begin{restatable}[Variance of the \classname]{thm}{univariance}
\label{theo:varianceunified}
Under the same conditions as in Theorem~\ref{theo:biasunified}, the variance of $\Rw(\pi)$ with weighting functions $\mathbf{w}=(\wva,\wvb,\wvg)$ is
\begin{equation}
    \begin{split}
     & \frac{1}{n}\Big\{ \mathbb{V}_x\Big(\mathbb{E}_{\pi}[ \wva \Delta - \wvb \zeta\delta + \wvg(\Delta-\zeta(\delta+\Delta))\\
        &+(\wva+\wvb+\wvg)\delta]\Big)+\mathbb{E}_{x}\mathbb{E}_{\pi}\Big[(\wvb)^2c(1-\zeta)^2\sigma^2_r\Big]\\
        &+\mathbb{E}_{x}\Big[\mathbb{V}_{\pi_0}(\wvb c(1-\zeta)\delta+\wvg c(1-\zeta)(\delta+\Delta))\Big]\Big\}\\
    \end{split}
\end{equation}
 \end{restatable}
\textit{Proof:} See Appendix~\ref{Appendix: prooftheo2}.

Throughout the rest of the paper, these general results will guide the design of new estimators, and they will allow us to compare different estimators within the family with respect to their bias/variance trade-offs.

\section{\blendname Estimator (\blendabbrNS)}

In this section, we identify a new estimator within this family that has many desirable properties. In particular, the estimator arises naturally by filtering our family of estimators according to these properties. First, we would like the estimator to be unbiased if the reward model and the propensity model are correct, which can be achieved through constraining the weights $(\wa,\wb,\wg)$ to sum to $1$ for each context-action pair. Second, the estimator should be applicable to a wide range of partial information settings, including learning to rank, which requires $\wg=0$. Third, we would like to achieve low MSE, which argues for data-dependent weights that allow an instance dependent trade-off between bias and variance. And, fourth, we would like to use the estimator for gradient-based learning, which implies that the weighting functions need to be (sub-)differentiable.

These desiderata and constraints lead us to the following new \blendname estimator (\blendabbrNS).
\begin{eqnarray*}
\Rcab(\pi) = \Rw(\pi) \mbox{ with} \left\{\!\!\!
    \begin{array}{l}
    \wa = 1- \min\!\left\{M\frac{ \pi_0(\newy|x_i)}{\pi(\newy|x_i)},1\right\} \\
    \wb = \min\!\left\{M\frac{ \pi_0(y_i|x_i)}{\pi(y_i|x_i)},1\right\} \\
    \wg = 0 
    \end{array}\right. 
\label{eq:cab}
\end{eqnarray*}
It is easy to see that \blendabbr interpolates between DM and IPS in an example-dependent way, which allows trade-off between bias and variance by controlling $M$. In particular, \blendabbr inherits the idea that clipping is a data-dependent way to achieve smaller MSE. However, \blendabbr imputes a regression estimate $\hat{\delta}(x_i,\newy)$ proportional to the clipped-off portion of the IPS weight -- unlike clipped IPS (see Table~\ref{tab:familyweights}) that implicitly imputes zero. Note that the particular choice of weights $(\wa,\wb,\wg)$ makes $\Rcab(\pi)$ continuous and subdifferentiable with respect to the parameters of the policy $\pi$.

\subsection{Bias and Variance Analysis}
\label{sec:bias_variance_analysis}
We now analyze the bias and variance of \blendabbrNS\ as an instance of the counterfactual estimator family, and we will compare them to those of IPS, cIPS, DR, and DM.

\begin{restatable}[Bias of \blendabbrNS]{thm}{bias}
\label{theo:biasblend}
For contexts $x_1, x_2, \cdots, x_n$ drawn i.i.d from some distribution $P(\mathcal{X})$ and for actions $y_i\sim \pi_0(\mY|x_i)$, under Condition~\ref{as:commonsupport} the bias of $\Rcab(\pi)$ is
\begin{align}
\mathbb{E}_{x}\mathbb{E}_{\pi}&\Big[-\delta\zeta\indicator{\{\hat{c}\leq M\}}\\
& + \!\{\Delta(1-\frac{M}{c(1\!-\!\zeta)})-\frac{M}{c(1\!-\!\zeta)}\delta\zeta\} \indicator{\{\hat{c}\!>\!M\}}\Big]\nonumber
\end{align}
\end{restatable}
\textit{Proof:} See Appendix~\ref{Appendix:prooftheo3}

Note that the first part of the bias results from the use of IPS when the IPS weight is small, while the second part results from the convex combination of IPS and DM when the IPS weight is large.
\begin{restatable}[Variance of \blendabbrNS]{thm}{variance}
\label{theo:varianceblend}
Under the same conditions as in Theorem~\ref{theo:biasblend}, the variance of $\Rcab(\pi)$ is
\begin{equation}
    \begin{split}
 & \frac{1}{n}\Big\{ \mathbb{V}_x\Big(\mathbb{E}_{\pi}[\delta -\delta\zeta\indicator{\{\hat{c}\leq M\}}\\
&+(\Delta(1-\frac{M}{c(1-\zeta)})-\frac{M}{c(1-\zeta)}\delta\zeta) \indicator{\{\hat{c}>M\}}]\Big)\\
        &+\mathbb{E}_{x}\mathbb{E}_{\pi}\Big[c(1-\zeta)^2\sigma^2_r\indicator{\{\hat{c}\leq M\}} + \frac{M^2}{c}\sigma^2_r\indicator{\{\hat{c}> M\}}\Big]\\
        &+\mathbb{E}_{x}\Big[\mathbb{V}_{\pi_0}(c(1-\zeta)\delta\indicator{\{\hat{c}\leq M\}}+M\delta\indicator{\{\hat{c}> M\}})\Big]\Big\}\\
    \end{split}
\end{equation}
\end{restatable}
\textit{Proof:} See Appendix~\ref{Appendix:prooftheo4}

The first term of the variance is due to the randomness in context $x$, the second term results from the randomness in the rewards compounded with {\em bounded} IPS weights. The third term is, in expectation, the variability in the expected reward compounded with {\em bounded} IPS weights.
\paragraph{Bias improvements over cIPS and DM.}
We can now compare the bias of CAB to that of cIPS, which we can derive
as a special case of Theorem~\ref{theo:biasblend} with $\hat{\delta}(x_i,\newy)=0$ for all $(x_i,\newy)$ pair.
Focusing on the case of logged propensities for conciseness and its real-world prevalence in online systems, this reduces to $Bias(\hat{R}_{cIPS}(\pi)) = \mathbb{E}_x\mathbb{E}_{\pi}[-\delta(1-\frac{M}{c})\indicator{\{c> M\}}]$ for cIPS and $Bias(\Rcab(\pi)) = \mathbb{E}_x\mathbb{E}_{\pi}[\Delta(1-\frac{M}{c})\indicator{\{c> M\}}]$ for CAB.
It can be seen that if we have a moderately good predictor of the expected reward $\delta(x,\newy)$, \blendabbr will have an advantage as long as the predictor is better than imputing the constant 0 everywhere. In practice, it is sensible to assume that the reward estimation error $\Delta(x,\newy)$ is substantially smaller than $\delta(x,\newy)$, such that \blendabbr enjoys a substantial amount of bias reduction. 

In comparison to DM, \blendabbr can also enjoy smaller bias when the propensity is known,  since $Bias(\hat{R}_{DM}(\pi)) = \mathbb{E}_x\mathbb{E}_{\pi}[\Delta]$ while $Bias(\Rcab(\pi)) = \mathbb{E}_x\mathbb{E}_{\pi}[\Delta (1-\frac{M}{c})\indicator{\{c> M\}}]$ which reflects that CAB incurs bias only on the clipped portion of the importance sample weights.
\paragraph{Variance improvements over IPS and DR.} \label{sec:variance_improvement}
Comparing the variance of CAB to that of IPS and DR, we again focus on the case of logged propensities. From Theorem~\ref{theo:varianceunified} we can deduce that the variance of IPS is
 \begin{equation}
    \begin{split}
  \frac{1}{n}\Big\{ \mathbb{V}_x\Big( \mathbb{E}_{\pi}[\delta]\Big)
      + \mathbb{E}_x\mathbb{E}_{\pi}\Big[c\sigma^2_r\Big] +\mathbb{E}_{x}\Big[\mathbb{V}_{\pi_0}(c\delta)\Big]\Big\}
    \end{split}
    \end{equation}
while the variance for \blendabbr is
\begin{equation}
    \begin{split}
     \frac{1}{n} & \Big\{  \mathbb{V}_x\Big( \mathbb{E}_{\pi}[\delta + \Delta(1-\frac{M}{c})\indicator{\{c>M\}}]\Big) \\
    & + \mathbb{E}_x\mathbb{E}_{\pi}\Big[c\sigma^2_r\indicator{\{c\leq M\}} + \frac{M^2}{c}\sigma^2_r\indicator{\{c>M\}}\Big]  \\
    & + \mathbb{E}_{x}\Big[\mathbb{V}_{\pi_0}[(c\delta)\indicator{\{c\leq M\}}+M\delta\indicator{\{c>M\}}]\Big]\Big\}.
    \end{split}
\end{equation}
The first term
is similar for both estimators since $\delta$ can be expected to dominate $\Delta$. The second and the third terms, which are the variance of the reward $r(x,\newy)$ and the expected reward $\delta(x_i, \newy)$ compounded with the IPS weights $c$, can be very large for IPS when the logging policy $\pi_0$ and the target policy $\pi$ are very different.
In contrast, for CAB these two terms are bounded by $M\mathbb{E}_{\pi}[\sigma^2_r]+M^2\mathbb{E}_{x}[\mathbb{V}_{\pi_0}(\delta)]$, and thus will be smaller than those for IPS.

Comparing CAB to DR, note that DR intends to reduce the variance of IPS by putting weight 1 on the observed loss term $\beta(x_i,y_i)$ and -1 on the estimated loss term $\gamma(x_i,y_i)$. However, this "residual term" is still compounded with the IPS weights $c$ and can blow up the variance either when  we have a poor estimate $\Delta$ or when the target policy is very different from the logging policy. This is apparent in the second and third terms of the variance of DR as derived from Theorem~\ref{theo:varianceunified}.
 \begin{equation}
    \begin{split}
    & \frac{1}{n}\Big\{ \mathbb{V}_x\Big( \mathbb{E}_{\pi}[\delta]\Big)
    +\mathbb{E}_x\mathbb{E}_{\pi}\Big[c\sigma^2_r\Big] +\mathbb{E}_x\Big[\mathbb{V}_{\pi_0}(c\Delta)\Big]\Big\}
    \end{split}
    \end{equation}
This is again different from \blendabbrNS, where the IPS weights are all bounded by $M$.



\subsection{CAB-DR}

Finally, we present a variant of CAB that incorporates the DR estimator, called CAB-DR. While this leads to an estimator that cannot be used for ranking, we investigate whether the control variate of DR leads to even better estimates than CAB. The key idea is to substitute the IPS part of CAB with the DR estimator, leading to
\begin{eqnarray*}
\Rcabdr(\pi) \!=\! \Rw(\pi) \mbox{ with} \left\{\!\!\!
    \begin{array}{l}
    \wa = 1 \\
    \wb = \min\!\left\{M\frac{ \pi_0(y_i|x_i)}{\pi(y_i|x_i)},1\right\} \\
    \wg = -\! \min\!\left\{\!M\frac{ \pi_0(\newy|x_i)}{\pi(\newy|x_i)},\!1\!\right\}
    \end{array}\right.
\label{eq:cab-dr}
\end{eqnarray*}
as a clipped version of DR. Using Theorem~\ref{theo:biasunified}, it is easy to derive the bias of CAB-DR to be
\begin{equation*}
\begin{split}
\mE_x\mE_{\pi}\Big[\zeta\Delta\indicator{\{\hat{c}\leq M\}}+
\Delta(1-\frac{M}{c}) \indicator{\{\hat{c}> M\}}\Big].
\end{split}
\end{equation*}
If the propensities are logged, then the bias terms for CAB and CAB-DR are identical. However, if the propensities are only approximates, CAB will suffer from more bias from the term $\mE_x\mE_{\pi}[\zeta\delta\indicator{\{\hat{c}\leq M\}}]$ compared to $\mE_x\mE_{\pi}[\zeta\Delta\indicator{\{\hat{c}\leq M\}}]$ for CAB-DR.

The variance of CAB-DR is given in Appendix~\ref{Appendix:prooftheo5}. Given logged propensities, the variance of CAB-DR differs from that of CAB in only a single term. For CAB-DR we have $$\mE_x\Big[\mV_{\pi_0}\Big(c(-\Delta)\indicator{\{c\leq M\}}+M(-\Delta)\indicator{\{c> M\}}\Big)\Big],$$
while for CAB we have
$$\mE_x\Big[\mV_{\pi_0}\Big(c(\delta)\indicator{\{c\leq M\}}+M(\delta)\indicator{\{c> M\}}\Big)\Big].$$
In general cases, by adopting the idea of choosing opposite weights for $\beta(x_i,y_i)$ and $\gamma(x_i,y_i)$ from DR, the third variance term for CAB-DR becomes $\mathbb{E}_x[\mathbb{V}_{\pi}(-\wvb c(1-\zeta)\Delta)]$, which is no longer on the order of $\mathbb{E}_x[\mathbb{V}_{\pi}(\wvb c(1-\zeta)\delta)]$.

\begin{figure*}[!htb]
\centering
\subfigure{
\begin{minipage}{0.24\linewidth}
 \centering
  \includegraphics[width=0.95\linewidth]{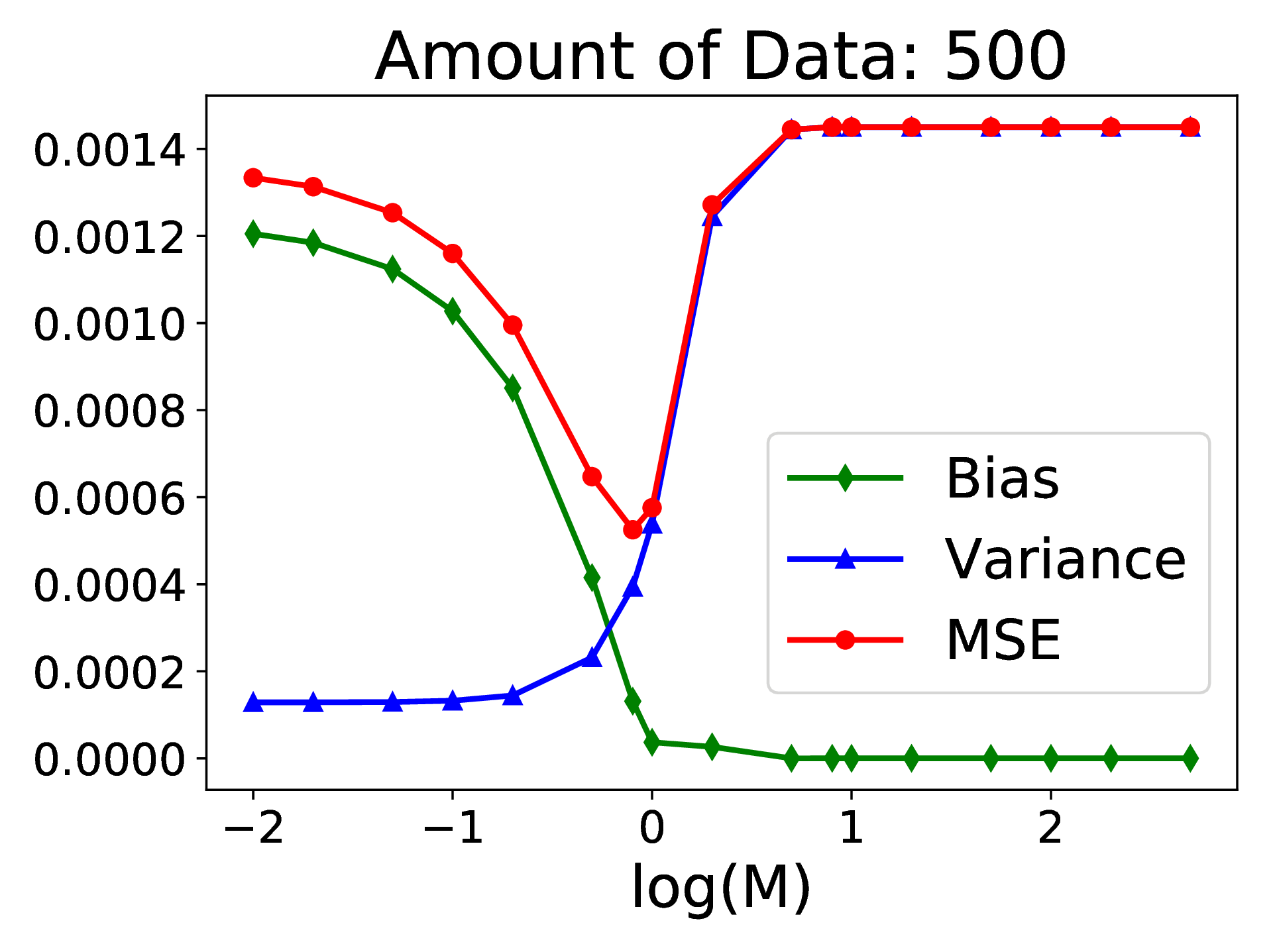}
 \label{fig:num500}
\end{minipage}}
\subfigure{
\begin{minipage}{0.24\linewidth}
\centering
  \includegraphics[width=0.95\linewidth]{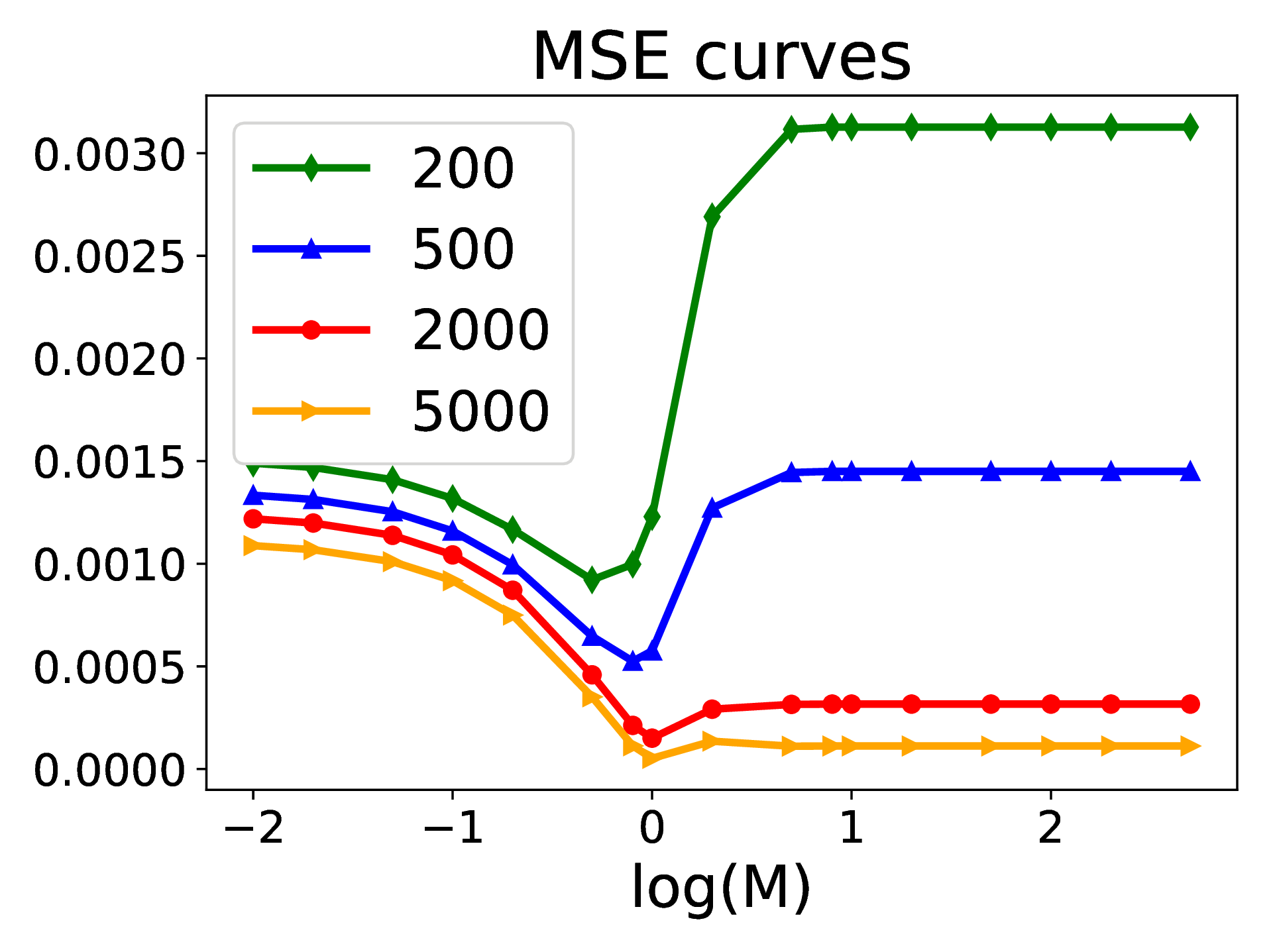}
\label{fig:num2000}
\end{minipage}}
\subfigure{
\begin{minipage}{0.24\linewidth}
\centering
   \includegraphics[width=0.95\linewidth]{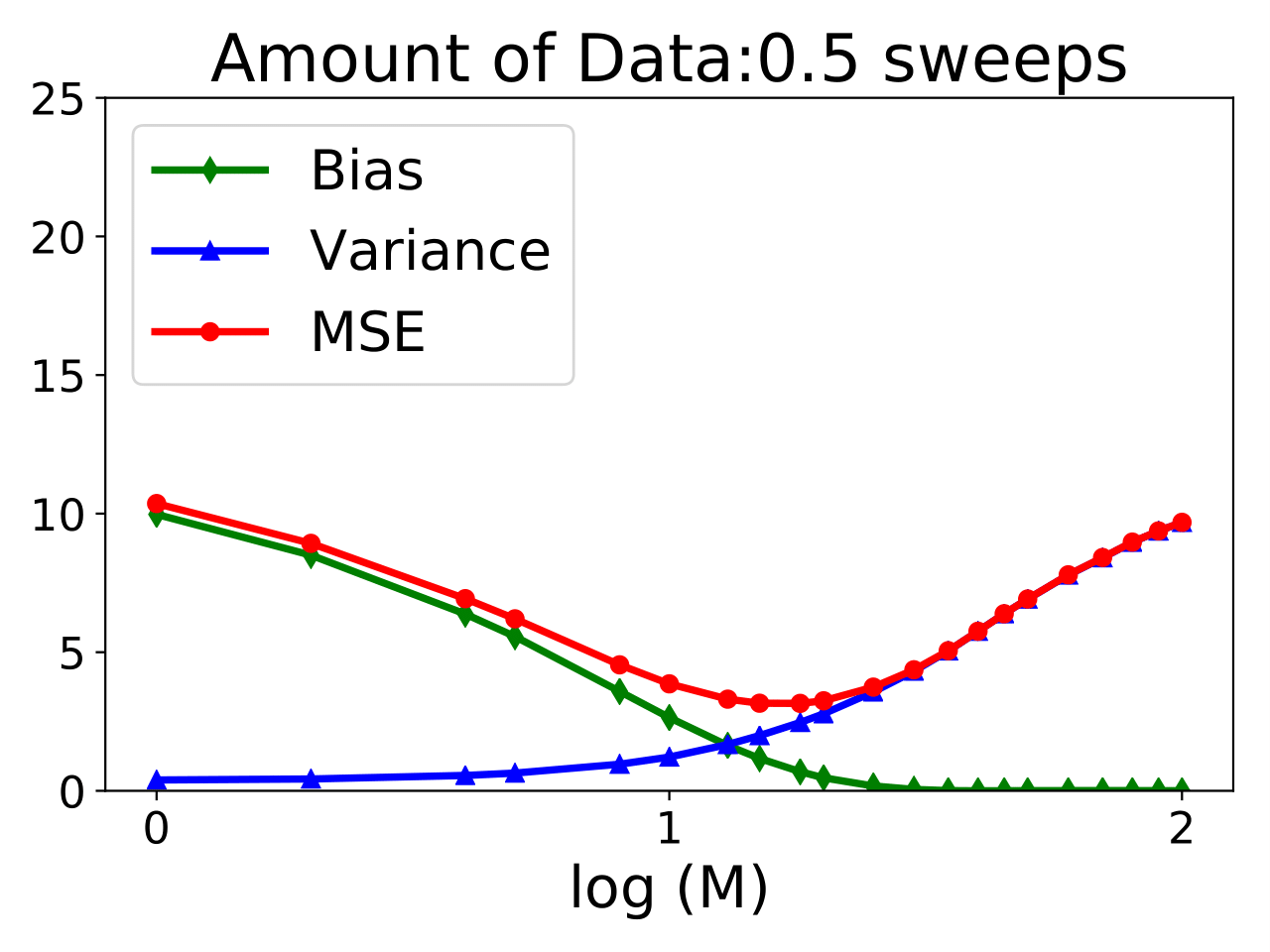}
\label{fig:0_5_sweeps}
\end{minipage}}
\subfigure{
\begin{minipage}{0.24\linewidth}
\centering
   \includegraphics[width=0.95\linewidth]{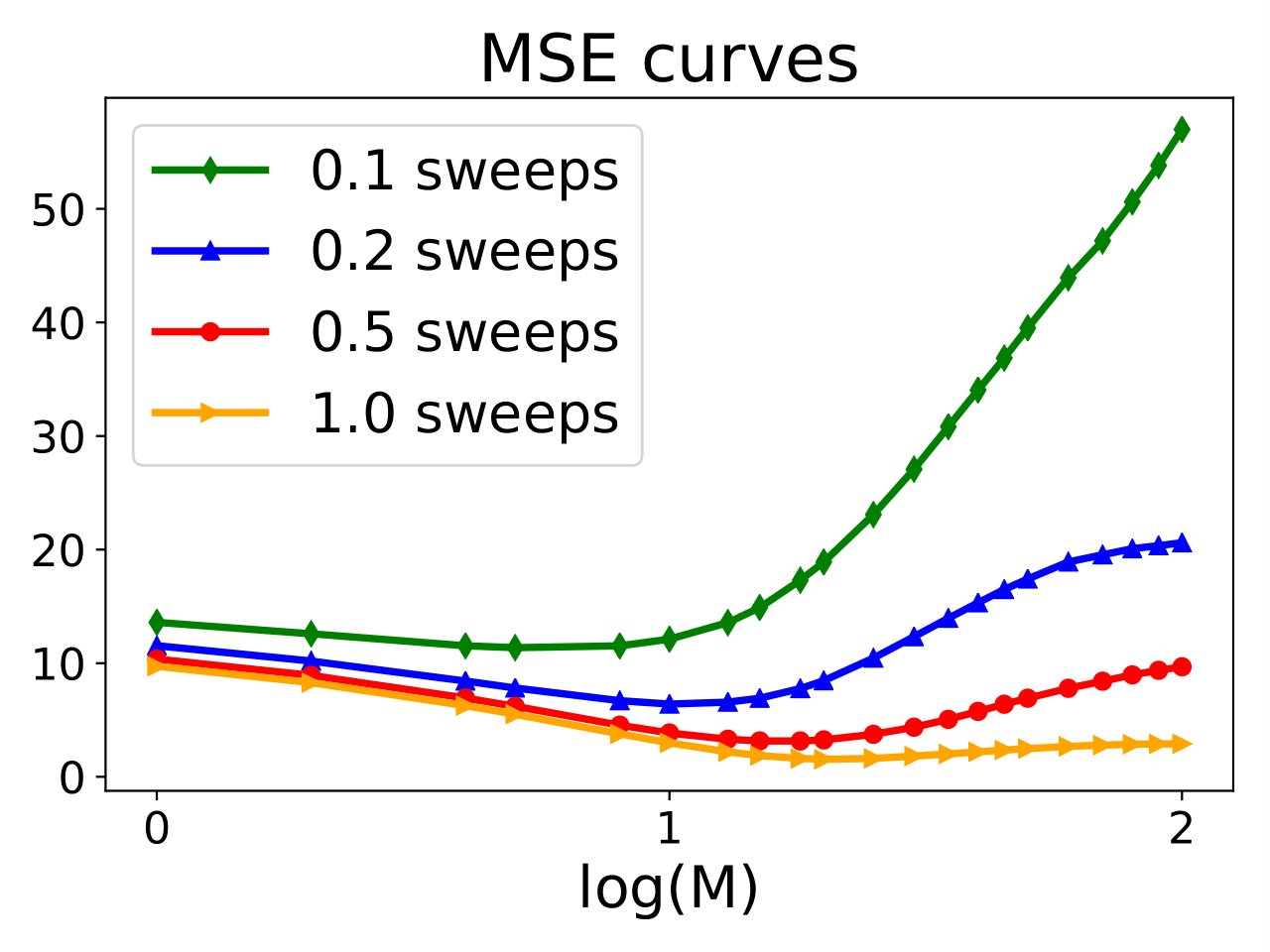} 
\label{fig:mse_dff_data}
\end{minipage}}
\caption{The Bias, Variance and MSE graph for \blendabbr on the \textsc{satimage} and \textsc{Yahoo! LTR} dataset. From left to right: (a) The bias, variance and MSE  curves for the \textsc{satimage} dataset. (b) MSE curves for \textsc{satimage} when we vary the amount of log data. (c) The bias, variance and MSE curves for the \textsc{Yahoo! LTR} dataset. (d) MSE curves for \textsc{Yahoo! LTR} when we vary the amount of log data. }\label{fig:biasvar}
\end{figure*}

\section{Experiments}
\label{sec:experiments}
We empirically examine the evaluation accuracy and learning performance of \blendabbr in two different partial-information settings.
In the BLBF setting, we conduct the experiments on bandit feedback data for multi-class classification. This setting is extensively used in the off-policy evaluation literature \cite{dudik2011doubly,wang2016optimal}. In the LTR setting, the experiments are based on user feedback with position bias for ranking \cite{joachims2017unbiased}. In both cases, we use real datasets from which we sample synthetic bandit or click data. This increases the external validity of the experiments, while at the same time providing ground truth for a bias/variance analysis. Furthermore, it allows us to vary the properties of both data and logging policy $\pi_0$ to explore the robustness of the estimators.

\subsection{Experiment Setup}
For the BLBF setting, our experiment setup follows \citet{dudik2011doubly} and \citet{wang2016optimal} using the standard supervised $\to$ bandit conversion \cite{agarwal2014taming} for several multiclass classification datasets from the UCI repository \cite{uci}. In the LTR setting, we follow the experiment setup of \citet{joachims2017unbiased} and conduct experiments on the \textsc{Yahoo! LTR} Challenge corpus (set 1), which comes with a train/validation/test split. More details are given in Appendix~\ref{sec:ltr}.

In both settings, we use a small amount of the full-information training data to train a logger $\pi_0$ and a regression model $\hat{\delta}$. The policy $\pi$ to be evaluated is trained on the whole training set. The partial feedback data is generated from the full-information test set. We evaluate the policy $\pi$ with different estimators on the partial feedback data of different sizes and treat the performance of the full-information test set as the ground truth. The performance is measured by the expected test error and the average rank of positive results for BLBF and LTR respectively. We repeat the experiments 500 times for BLBF and 100 times for LTR to calculate bias, variance and MSE.

For the BLBF learning experiments, we use POEM \cite{Swaminathan/Joachims/15c} to learn stochastic linear policies. For LTR, Appendix~\ref{sec:genpropsvm} derives a generalized version of propensity SVM-Rank \cite{joachims2017unbiased} that enables the use of CAB and other estimators with $\wg=0$ from our family. As input to the learning algorithms, different amounts of partial feedback data are simulated from the full-information training data. To avoid biases from the regression model, we adopt 90 percentile cIPS to conduct hyperparameter selection for $M$(or $\tau$ for SB) and regularization parameter on the partial feedback data simulated from the validation set. 
The experiments are run for 10 and 5 times on BLBF and LTR respectively and the average is reported. Details are shown in Appendix \ref{appendix: setup}.

\subsection{Experiment Results}

\paragraph{Can \blendabbr achieve improved estimation accuracy by trading bias for variance through $M$?}
We first verify that CAB can indeed achieve improved MSE by adjusting the bias-variance trade-off. Figure \ref{fig:biasvar} shows how the choice of $M$ affects bias and variance of \blendabbr on the \textsc{satimage} and \textsc{Yahoo! LTR} datasets using different amounts of data (qualitatively similar results are obtained for the other datasets). For each dataset, the bias decreases as we increase $M$ as expected, since \blendabbr moves towards the unbiased IPS estimator. However, the variance increases as more data points rely on IPS weighting. For all data set sizes, the best MSE always falls in the middle of the range of $M$ which confirms that \blendabbr can effectively trade-off between bias and variance through controlling $M$ for a range of data-set sizes. Moreover, the MSE curve suggests that the performance of CAB is pretty robust to the choice of $M$. 

\paragraph{How does \blendabbr compare to other off-policy estimators?}

\begin{figure*}[!htb]
\centering
\includegraphics[scale=0.055]{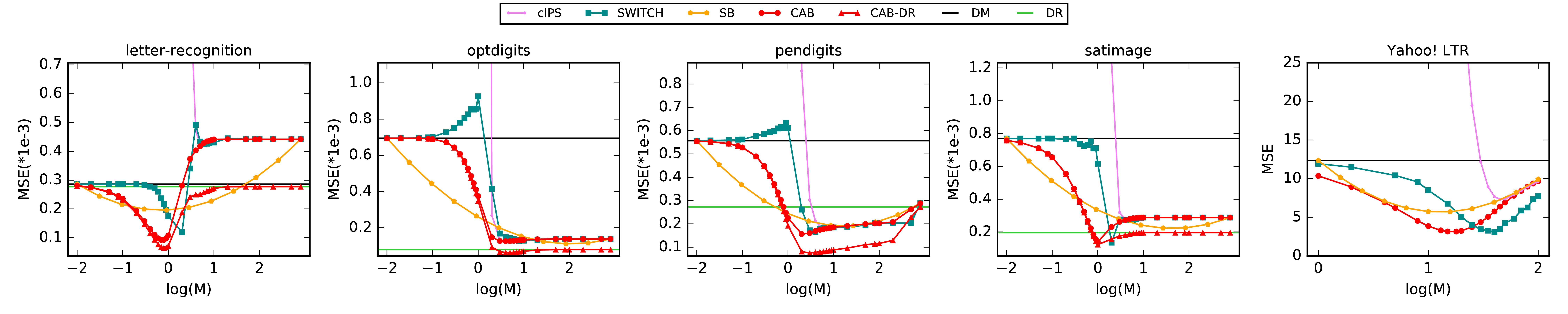}
\caption{MSE comparison of various off-policy estimators for different datasets}
\label{fig:compare}
\vspace*{-0.1cm}
\end{figure*}

Figure~\ref{fig:compare} compares different off-policy estimators on the 4 UCI datasets (selecting those with dataset size larger than 5000) and the \textsc{Yahoo! LTR} dataset. For each UCI dataset, we keep the test data size as 2,000. For the LTR dataset, we present the results with 0.5 sweeps of the test set. Notice that \blendabbr, \blendabbrNS-DR, cIPS and SWITCH all have a clipping parameter $M \in [0, \infty)$, while for SB the blending is achieved through the static weight parameter $\tau \in [0,1]$. To be able to plot SB together with the other estimators, we rescale $\tau$ in the plot for comparison. DM and DR do not have any hyperparameter, so we use two horizontal lines to represent them.

For both SB and \blendabbr (\blendabbrNS-DR), and across all datasets, we observe a $U$-shape curve for MSE with the optimum value in the middle. However, on most datasets \blendabbr (and \blendabbrNS-DR) substantially outperforms SB. Furthermore, \blendabbr outperforms cIPS in the full range of M on all datasets, indicating that imputing a reasonably good regression estimate is indeed consistently better than naively imputing zero. For the SWITCH estimator, the MSE curve is somewhat more erratic than that of \blendabbr especially on the UCI datasets, which we conjecture is due to the hard switch it makes and the discontinuities this implies.  While SWITCH can be used in LTR algorithms like Propensity SVM-Rank \cite{joachims2017unbiased} (details are shown in Appendix ~\ref{sec:ltr}), we show in the next section that this behavior may make model selection during learning more stable for \blendabbr than for SWITCH. Furthermore, \blendabbr performs at least comparable to SWITCH across all datasets, and on some it can be substantially more accurate than SWITCH.

For all datasets, DR outperforms IPS and DM as expected. However, \blendabbr (also \blendabbrNS-DR) still outperforms DR on most datasets, which validates the idea that estimators outside the class of unbiased estimators can have advantages on this problem.
Furthermore, when used in learning, one already faces a bias/variance trade-off due to the capacity of the policy space, such that it seems unjustified to insist on the unbiasedness of the empirical risk to begin with.

\paragraph{How robust is CAB on real-world data?}

We evaluated CAB on data from a contextual bandit problem at Amazon Music. Both the logging policy and the target policy are a Thompson sampling contextual bandit algorithm for which we estimated the respective policy $\pi^t_0(y|x)$ and $\pi^t_{target}(y|x)$ at each time step $t$ through Monte Carlo sampling. When analyzing the logging distribution, we found that the logging policy does not provide full support as the Thompson sampler converges, and on average only 33\% of the available actions have non-zero support. We used the model learned by the Thompson sampler of the logging policy as the regression imputation model, $\hat{\delta}(x,y)$.

\begin{figure}[t]
    \centering
     \includegraphics*[width=0.95\linewidth]{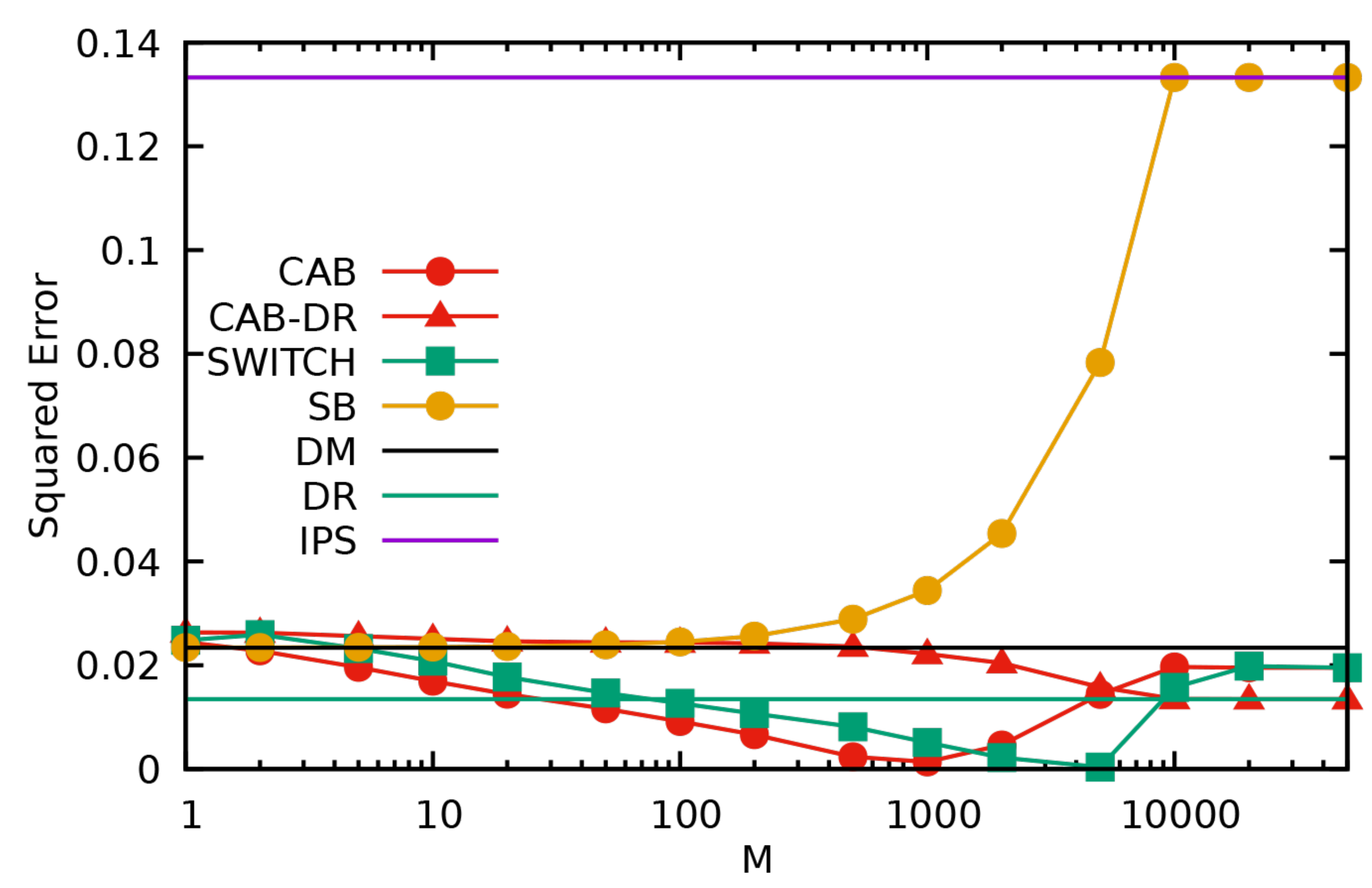}
    \vspace*{-0.4cm}
    \caption{Error of the estimates on the Amazon Music contextual bandit problem.}
    \label{fig:icml}
\end{figure}

\begin{figure*}[!htb]
\centering
\subfigure{
\begin{minipage}{0.24\linewidth}
 \centering
  \includegraphics[width=1.1\linewidth]{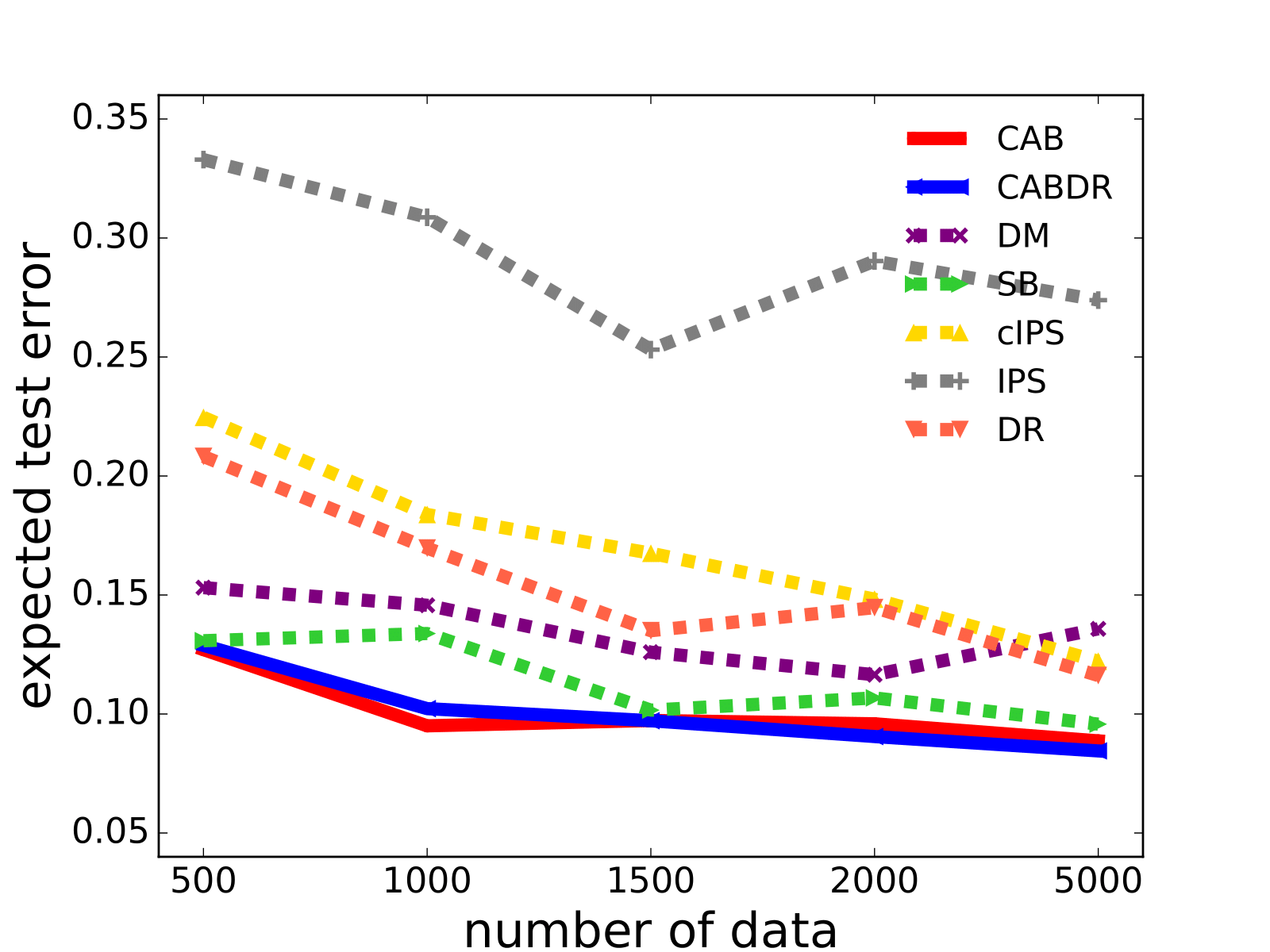}
 \label{fig:uci_learning}
\end{minipage}}
\subfigure{
\begin{minipage}{0.24\linewidth}
\centering
 \includegraphics[width=1.1\linewidth]{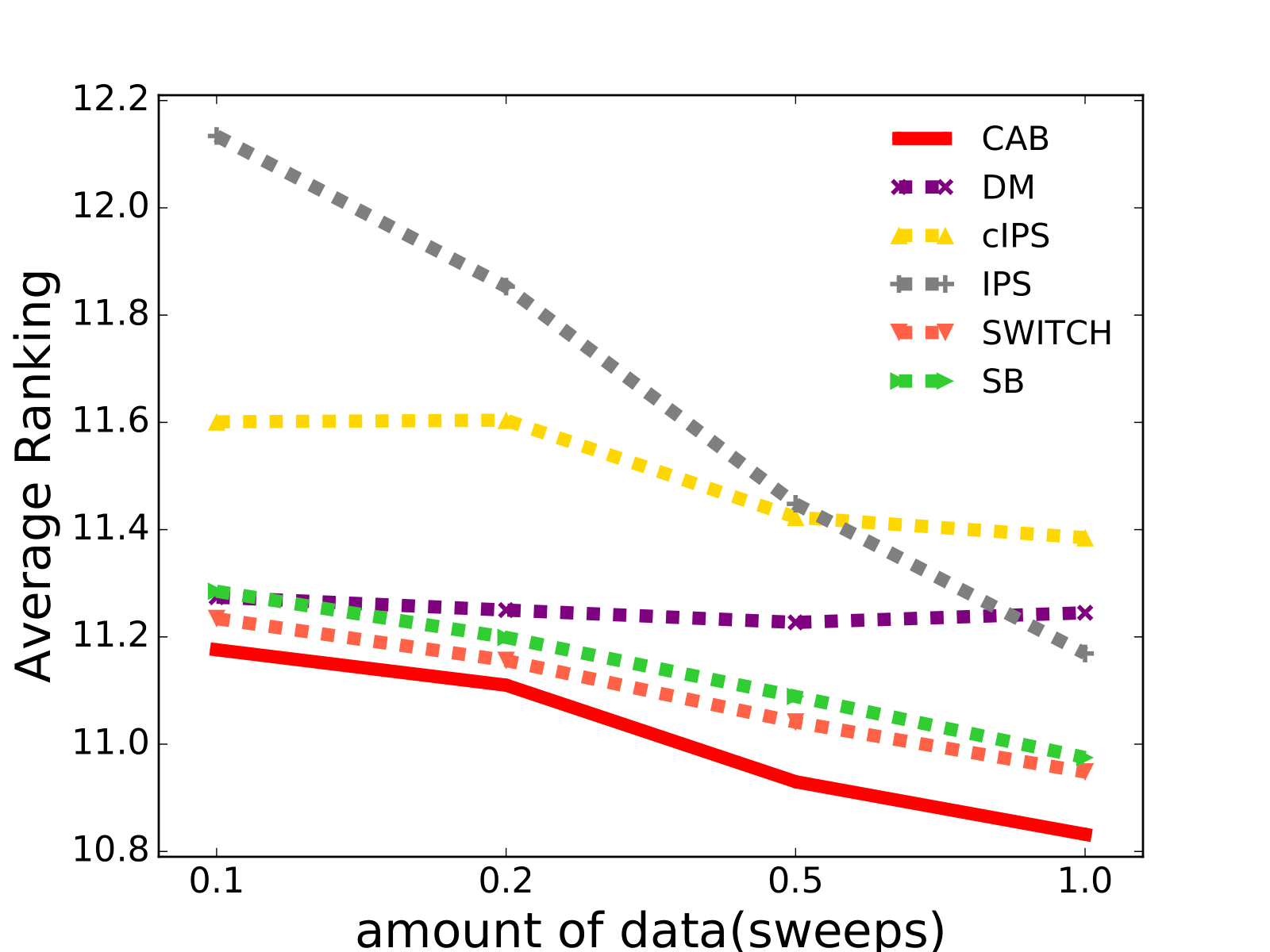}
\label{fig:ltr_learning}
\end{minipage}}
\subfigure{
\begin{minipage}{0.24\linewidth}
\centering
   \includegraphics[width=1.1\linewidth]{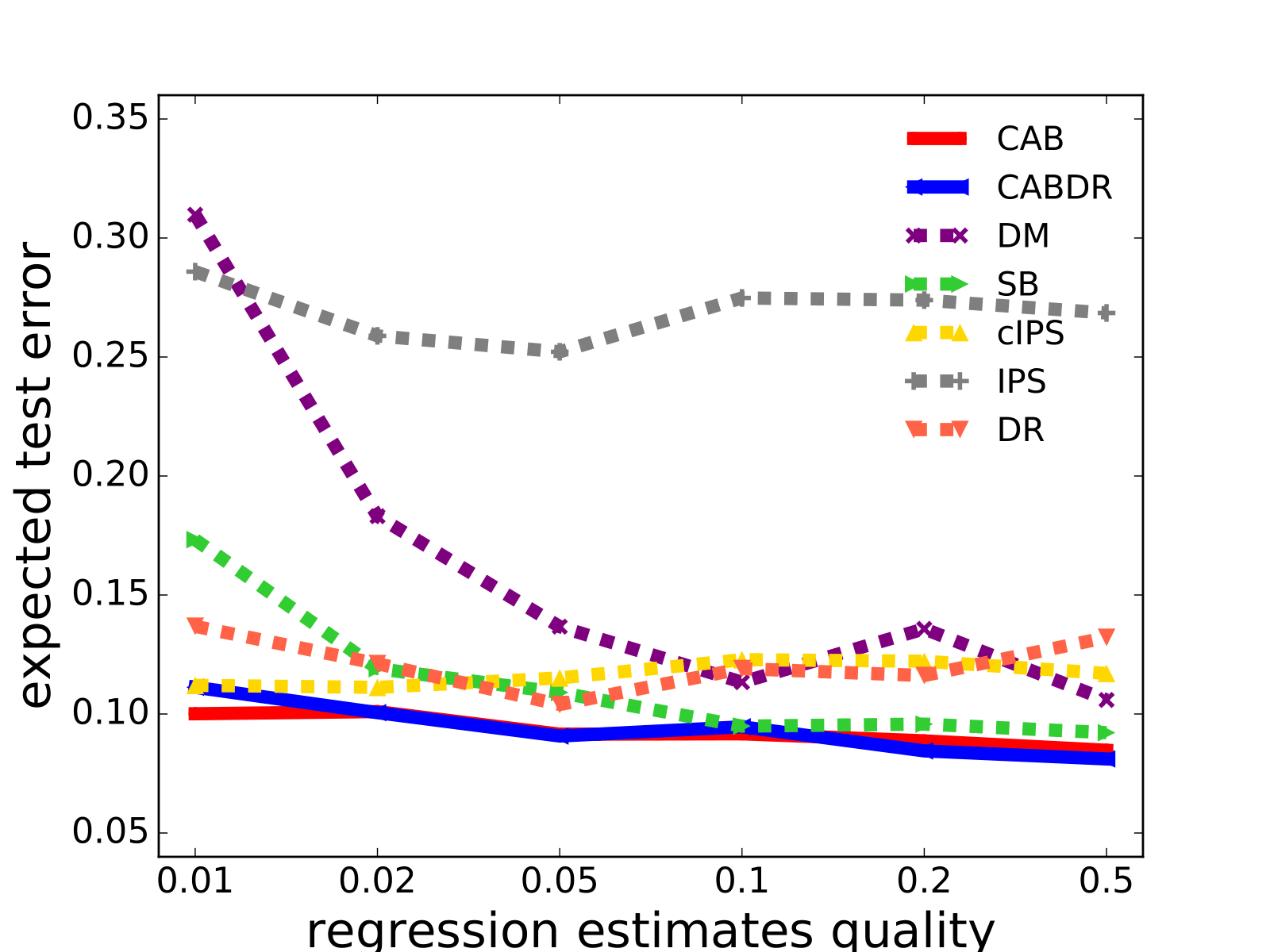}
\label{fig:reg}
\end{minipage}}
\subfigure{
\begin{minipage}{0.24\linewidth}
\centering
    \includegraphics[width=1.1\linewidth]{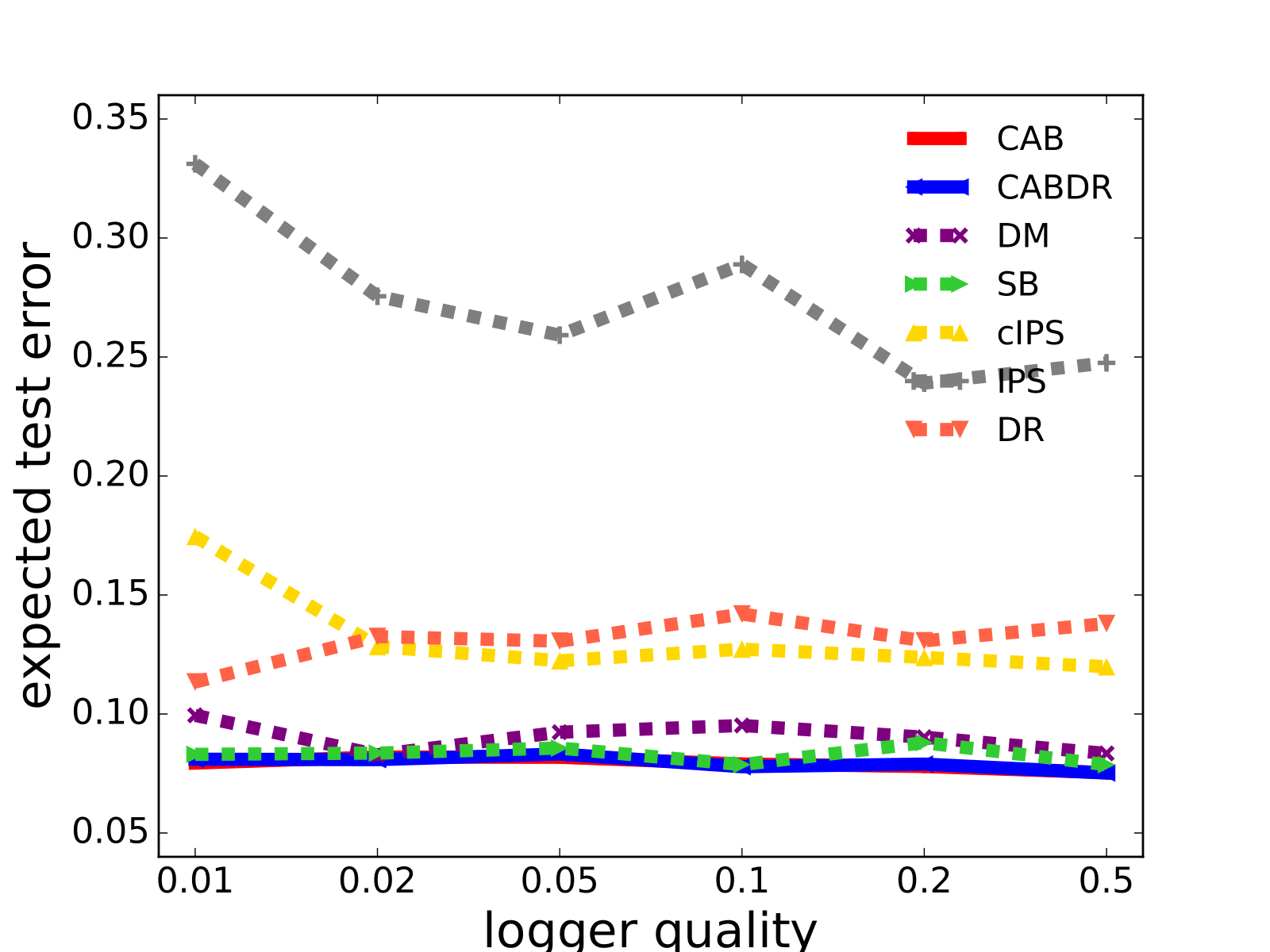}
\label{fig:logger}
\end{minipage}}
\vspace{-0.5cm}
\caption{Test set learning performance under various scenarios. (1)-(2) Performance vs. amount of data for \textsc{PENDIGITS} and \textsc{Yahoo! LTR}. (3) Performance vs. regression model quality for \textsc{PENDIGITS}. (4) Performance vs. logger quality for \textsc{PENDIGITS}.  }\label{fig:learning}
\end{figure*}

Figure~\ref{fig:icml} shows the error of the estimates depending on the clipping constant $M$, where we use the average reward of $\pi^t_{target}(y|x)$ measured during online A/B testing as the gold standard. For the majority of the values of the clipping constant  $M$, CAB and SWITCH show a higher level of accuracy than DM, IPS and SB. Both IPS and SB perform poorly. This is due to the fact that IPS has no mechanism for detecting and correcting for the missing support of the logging policy, while SB has a static blending constant. For a large range of $M$, CAB and SWITCH also outperform DR and CAB-DR. Overall, we find that both CAB and SWITCH outperform the other methods and can provide robust solutions to the contextual bandit problem at Amazon Music. 

\paragraph{How effective is learning with CAB as empirical risk across datasets?}

Table~\ref{datasets-table} shows the learning performance when the various estimators are used as empirical risk in POEM and Propensity SVM-Rank. We use the same datasets as in the evaluation experiments. For the 4 UCI datasets (with $5000$ training data for each) we report the average test set error, while for \textsc{Yahoo! LTR} (with $1$ sweep of data) we report the average rank of the relevant results for the test queries. For both lower is better. 
Across all datasets, learning with CAB performs very well, providing one of the best prediction performances on all datasets which validates that learning with an improved estimator will also lead to improved learning performance. Comparing CAB with CAB-DR, we don't see a real benefit in using the DR model in CAB-DR instead of using IPS in CAB. We conjecture that it is due to the over-reliance on the regression model, as DR relies more on it than IPS.

\begin{table}[ht]
\centering
\caption{Test set learning performance on various datasets.}
\label{datasets-table}
\begin{small}
\begin{sc}
\scalebox{0.7}
{
\begin{tabular}{lccccccr}
\toprule
DATA& Letter  &Optdigits &Satimage & Pendigits & Yahoo! LTR\\
\midrule
DM  & 0.6372 & 0.0649 & 0.3083 & 0.1133  & 11.25\\
DR  & 0.6852 & 0.0471 & 0.2762 & 0.1191 & -\\
IPS  & 0.8969  & 0.0695  & 0.3266 & 0.2748 & 11.17\\
cIPS & 0.8504 & 0.0447 & \textbf{0.2415} & 0.1228 & 11.39\\
SB & 0.6091  & 0.0460 & 0.2481 & 0.0949 & 10.98\\
SWITCH & - & - & - & - & 10.95\\
CAB & \textbf{0.5740} & \textbf{0.0445}  & 0.2442 & \textbf{0.0917} & \textbf{10.83}\\
CAB-DR &0.5877 & 0.0461 & 0.2762 & 0.0946 & -\\
\bottomrule
\end{tabular}}
\end{sc}
\end{small}
\vskip -0.2in
\end{table}

\paragraph{How does learning performance change when we increase the amount of training data?}
Plots (1) and (2) in Figure~\ref{fig:learning} show test set performance when we increase the amount of training data for \textsc{pendigits} and \textsc{Yahoo! LTR}. 
We find that performance improves for all estimators as the amount of data increases. However, for all training data sizes, we observe that CAB and CAB-DR perform very well compared to the other estimators, with a substantial improvement over IPS, cIPS and DR especially in the small sample setting. Furthermore, \blendabbr performs better than SWITCH in the ranking setting, which we attribute to the more erratic behavior of SWITCH during model selection.

\paragraph{How does learning performance change when we vary the quality of the estimated regression model?}
In order to vary the quality of the regression model used by DM, DR, CAB, CAB-DR, SWITCH and SB, we vary the fraction of full-information data that is used to learn the regression model. 
The results are shown in plot (3) of Figure~\ref{fig:learning}.
IPS and cIPS result in two flat lines (up to variance), as they do not rely on the regression model. The other methods improve with the quality of the regression model, and both CAB and CAB-DR do well over the whole range.

\paragraph{How does learning performance change when we vary the quality of the logging policy?}
Plot (4) of Figure~\ref{fig:learning} shows the results when changing the logger quality. To do so, we used different fractions of the full-information data to train the logging policy, from 0.01 to 0.5.
Since the DM estimator is independent of the logging policy, it results in a flat line (up to variance).
IPS and cIPS are heavily affected by the logger quality, due to their dependency on the IPS weights, while 
CAB, CAB-DR and SB are only moderately affected. Overall, CAB and CAB-DR perform well across the whole range.

\section{Conclusion}
This paper proposed a parametric family of estimators for off-policy evaluation, which unifies and characterizes a number of popular off-policy estimators. The theoretical analysis also motivates the \blendabbr estimator, which not only provides a controllable bias/variance trade-off for off-policy evaluation, but is also continuous with respect to the target policy to enable gradient-based learning. We argue theoretically that \blendabbr can be less biased than cIPS and DM and often enjoys smaller variance than IPS and DR. Experiment results on two different partial-information settings -- contextual bandit and partial-information LTR -- confirm that \blendabbr can consistently achieve improved evaluation performance over other counterfactual estimators, and that it also leads to excellent learning performance.

\section*{Acknowledgements} 
This research was supported in part by NSF Awards IIS-1615706, IIS-1513692, and 1740822 as well as a gift from Amazon. All content represents the opinion of the authors, which is not necessarily shared or endorsed by their respective employers
and/or sponsors.



\bibliography{example_paper}
\bibliographystyle{icml2019}
\newpage
\onecolumn
\appendix

\section{Appendix: Counterfactual Learning to Rank}
\label{sec:ltr}
We also consider another important partial information setting: ranking evaluation and learning to rank based on implicit feedback (e.g. clicks, dwell time). Here the selection bias on the feedback signal is strongly influenced by position bias, since items lower in the ranking are less likely to be discovered by the user. However, it can be shown that this bias can be estimated \cite{joachims2017unbiased,wang2018position,agarwal2019estimating}, and that the resulting estimates can serve as propensities in IPS-style estimators.

To connect the ranking setting with the contextual bandit setting more formally, now each context $x \sim P(\mathcal{X})$ represents a query and/or user profile. Given ranking function $\pi$, we use $\pi(x)$ to represent the ranking for query $x$. 
However, in the LTR setting, we no longer consider the actions atomic, but instead treat rankings as combinatorial actions where the reward decomposes as a weighted sum of component rewards. Formally speaking, the reward for rankings $\pi(x)$ is denoted as $\Delta(\pi(x)|x,r):=\sum_{d\in \mathbf{d}}\lambda(rank(d|\pi(x)))r(x,d)$, where $\lambda(\cdot)$ is a function that maps a rank to a score, $\mathbf{d}$ is the candidate set for query $x$ and $rank(d|\pi(x))$ represents the rank of document $d$ in the candidate set $\mathbf{d}$ under the ranking policy $\pi$ given context $x$, and $r(x,d)\in\{0,1\}$ is the relevance indicator.
Then we can define the overall reward for a ranking policy $\pi$ as
\begin{equation}
    R(\pi) = \int\Delta(\pi(x)|x,r)dP(x)
\end{equation}
Note that in this partial information setting we typically do not observe rewards for all (query, document) pairs. The only observable reward per component may be whether the user clicks the document or not, $c(x,\pi_0(x),d)\in\{0,1\}$, and there is inherent ambiguity whether the lack of a click means lack of relevance or lack of discovery. Here we use a latent variable $o(x,\pi_0(x),d)\in\{0,1\}$ to represent whether the user $x$ observes document $d$ under the logging policy $\pi_0$, which then leads to the following click model: a user clicks a document when the user observes it and the document is relevant, $c(x,\pi_0(x),d) = r(x,d)\cdot o(x,\pi_0(x),d)$.
Note that the examination $o(x,\pi_0(x),d)$ is not observed by the system, but one can estimate a missingness model \cite{joachims2017unbiased}, and use $p(x,\pi_0(x),d)$ be the (estimated) probability of $\indicator{\{o(x,\pi_0(x),d)=1\}}$. We denote this
probability value as the propensity of the observation.
In practice one could estimate the propensities as outlined in \cite{agarwal2019estimating, Fang/etal/19}. The logged data we get is in the format of $\mathcal{S}=\{\{(x_i,d_{ij},p_{ij},c_{ij})\}_{j=1}^{m_{i}}\}_{i=1}^n$ where $m_i$ is the number of candidates for context $x_i$, $p_{ij}$ is $p(x_i,\pi_0(x_i),d_{ij})$ and $c_{ij}$ is $c(x_i,\pi_0(x_i),d_{ij})$. For additive ranking metrics, various estimators in the \classname\ can be written in the form
\begin{equation*}
\begin{split}
\Rw(\pi)=\frac{1}{n} \sum_{i=1}^n\sum_{j=1}^{m_i}&\bigg[w_{ij}^{\alpha}\alpha_{ij}+o_{ij}w_{ij}^{\beta}\beta_{ij} + o_{ij}w_{ij}^{\gamma} \gamma_{ij} \bigg]\cdot
\lambda(rank(d_{ij}|x_i,\pi(x_i)))\\
\hspace*{-0.5cm}\mbox{with}\quad
\alpha_{ij}:= &\hat{\delta}(x_i,d_{ij}),
\beta_{ij}  := \frac{r_{ij}}{p_{ij}},  \gamma_{ij}:= \frac{\hat{\delta}(x_i,d_{ij})}{p_{ij}}.
\end{split}
\end{equation*}
where $w_{ij}^{\alpha}$, $w_{ij}^{\beta}$ and $w_{ij}^{\gamma}$ are the weight functions of the three components of the \classname. 
 Given a perfect reward model and logged propensities, we can get unbiased estimate of additive ranking metrics if the weight functions sum to 1. The weights of various estimators in this setting are shown in Table~\ref{tab:LTR_weights}.

\begin{table}[ht]
    \centering
    \caption{The weight functions for different estimators of the family $\Rw(\pi)$ in the ranking setting.}
    \label{tab:LTR_weights}
    \setlength{\tabcolsep}{4pt}
    \begin{tabular}{|l|c|c|c|} \hline
    Estimator & $w_{ij}^{\alpha}$ & $w_{ij}^{\beta}$ & $w_{ij}^{\gamma}$ \\ \hline
    DM    &  1 & 0 & 0 \\
    IPS   &  0 & 1 & 0 \\
    cIPS  &  0 & $\min\!\left\{\!M p_{ij},\!1\!\right\}$ & 0 \\
    SB    &  $1-\tau$ & $\tau$ & 0 \\
    SWITCH&  $\mathds{1}{\left\{\!\frac{1}{p_{ij}}\!>\!M\! \right\}}$ & $\mathds{1}{\left\{\!\frac{1}{p_{ij}}\!\leq\! M \!\right\}}$ & 0 \\
    CAB&  $1-\min\!\left\{\!M p_{ij},\!1\!\right\}$ & $\min\!\left\{\!M p_{ij},\!1\!\right\}$ & 0 \\ \hline
    \end{tabular}
\end{table}


Note that estimators with $w^{\gamma}_{ij}\neq0$ (DR, CAB-DR) are not applicable in this setting since the third term $o_{ij}w_{ij}^{\gamma} \gamma_{ij}$ depends on $o_{ij}$, which is not observed nor fully captured by $c_{ij}$. However, the second term is computable since the unobserved $o_{ij}$ and $r_{ij}$ are captured by $c_{ij}$ through $o_{ij}\beta_{ij} = \frac{o_{ij}r_{ij}}{p_{ij}}=\frac{c_{ij}}{p_{ij}}$.
The SWITCH estimator is applicable for learning in this setting since the weights of the estimator do not depend on the ranking policy to be learned.

\section{Appendix: Proofs}
In this appendix, we provide proofs of the main theorems.
\subsection{Proof of Theorem 1}
\label{Appendix: prooftheo1}
\unibias*
\begin{proof}
For simplicity, we make the following notations throughout the proof. We let $\zeta:=\zeta(x,y)$ denote the multiplicative deviation of the propensity estimate from the true propensity model, and $\Delta:=\Delta(x,y)$ be the additive deviation of the reward model from the true reward. Recall
\begin{equation}
\zeta(x,y) = 1 -\frac{\pi_0(y|x)}{\hat{\pi}_0(y|x)}
\end{equation}
\begin{equation}
\Delta(x,y) = \hat{\delta}(x,y)-\delta(x,y).
\end{equation}
Moreover, the $\sigma^2_r:=\sigma^2_r(x,y)$ is used to denote the randomness in reward $r(x,y)$ with $\sigma^2_r(x,y) = \mV_r(r(x,y)|x,y)$. Moreover, we denote the true IPS weight $\frac{\pi(y|x)}{{\pi_0}(y|x)}$ as $c(x,y)$ with the estimated version being $\hat{c}(x,y)$.
Also, let $\wva:= \wxya, \wvb:=\wxb$ and $\wvg:=\wxg$ be the abbreviation for the weighting functions.

We will start the proof by calculating the expectation of three different components of $\Rw(\pi)$. For the $\alpha_{i\newy}$ component, this term is independent of the distribution of $y_i$, and we have:
\begin{equation}
\begin{split}
\mE\Big[\frac{1}{n}\sum_{i=1}^n\sum_{\newy \in \mathcal{Y}}\pi(\newy|x)\wa\alpha_{i\newy}\Big] = \mE_x\Big[\sum_{\newy\in\mathcal{Y}}\pi(\newy|x)\wxa\hat{\delta}(x,\newy)\Big] = \mE_x\mE_{y\sim\pi}\Big[\wxya(\delta+\Delta)\Big]
\end{split}
\end{equation}
For the IPS term $\beta_i$, with $x_i\sim P(\mX)$ and $y_i\sim \pi_0(\mY|x)$.
\begin{equation}
    \begin{split}
    \mE\Big[\frac{1}{n}\sum_{i=1}^n\pi(y_i|x_i)\wb\beta_i\Big]
        & = \mE_x\mE_{y\sim \pi_0}\mE_r\Big[\wxb\frac{\pi(y|x)}{\hat{\pi}_0(y|x)}r(x,y)\Big]\\
        & = \mE_x\mE_{y\sim \pi_0}\Big[\wxb \frac{\pi(y|x)}{\pi_0(y|x)}\frac{\pi_0(y|x)}{\hat{\pi}_0(y|x)}\delta\Big]\\
        & =  \mE_x\mE_{y\sim \pi_0}\Big[c\wxb(1-\zeta)\delta\Big]\\
        & = \mE_x\mE_{y\sim \pi}\Big[\wxb(1-\zeta)\delta\Big]
    \end{split}
\end{equation}
where the second equation follows from the fact that conditioning on $(x,y)$, $\mE_r[r(x,y)|x,y]=\delta(x,y)$.
For the third term $\gamma_i$, we have
\begin{equation}
    \begin{split}
        \mE\Big[\frac{1}{n}\sum_{i=1}^n\pi(y_i|x_i)\wg \gamma_i\Big]
        &= \mE_x\mE_{y\sim \pi_0}\Big[\wxg\frac{\pi(y|x)}{\hat{\pi}_0(y|x)}\hat{\delta}(x,y)\Big]\\
        & = \mE_x\mE_{y\sim \pi_0}\Big[\wxg\frac{\pi(y|x)}{\pi_0(y|x)}\frac{\pi_0(y|x)}{\hat{\pi}_0(y|x)}(\delta+\Delta)\Big]\\
        & =  \mE_x\mE_{y\sim \pi_0}\Big[c\wxg(1-\zeta)(\delta+\Delta)\Big]\\
        & =  \mE_x\mE_{y\sim \pi}\Big[\wxg(1-\zeta)(\delta+\Delta)\Big]
    \end{split}
\end{equation}
Combining these three terms and using the formula that $Bias(\Rw(\pi)) = \mE[\Rw(\pi)] - \mE_x\mE_{y\sim \pi}\mE_r[r]$, we have
\begin{equation}
    \begin{split}
        Bias(\Rw(\pi))  = \mE_x\mE_{y\sim\pi}\Big[\wva\Delta - \wvb\zeta\delta + \wvg(\Delta-\zeta(\delta+\Delta))+(\wva+\wvb+\wvg)\delta - \delta\Big]
    \end{split}
\end{equation}
\end{proof}

\subsection{Proof of Theorem 2}
\label{Appendix: prooftheo2}
\univariance*
\begin{proof}
We follow the same notation as in Appendix~\ref{Appendix: prooftheo1}. Let $\Rwi(\pi):= \sum_{\bar{y}\in\mathcal{Y}} \pi(\bar{y}|x_i) \: w_{i\bar{y}}^{\alpha} \: \alpha_{i\bar{y}} + \pi(y_i|x_i) \: w_i^{\beta} \: \beta_i
+  \pi(y_i|x_i) \: w_i^{\gamma} \: \gamma_i$ with the abbreviated version defined as $\rp$, and it is easy to see that $\mV(\Rw(\pi))=\frac{1}{n}\mV(\Rwi(\pi))$.
\begin{equation}
\begin{split}
    \mV(\Rwi(\pi)) & = \mV_x\Big(\mE_{y\sim\pi_0,r}[\rp|x]\Big) + \mE_x\Big[\mV_{y\sim\pi_0,r}(\rp|x)\Big] \\
    & = \mV_x\Big(\mE_{y\sim\pi_0,r}[\rp|x]\Big) + \mE_x\Big[\mE_{y\sim\pi_0}[\mV_r(\rp|x,y)|x]\Big] + \mE_x\Big[\mV_{y\sim\pi_0}(\mE_r[\rp|x,y]|x)\Big]
\end{split}
\end{equation}
For the first term, using the bias formula in Appendix~\ref{Appendix: prooftheo1}, it is easy to see that
\begin{equation}
    \mV_x\Big(\mE_{y\sim\pi_0,r}[\rp|x]\Big) = \mV_x\Big(\mE_{y \sim\pi}[\wxya\Delta - \wxb\zeta\delta + \wxg(\Delta-\zeta(\delta+\Delta))+(\wxa+\wxb+\wxg)\delta|x]\Big)
\end{equation}
For the second term, we will calculate $\mV_r(\rp|x,y)$ first.
\begin{equation}
\begin{split}
  \mV_r(\rp|x,y) & = \mV_r\Big(\wxb\frac{\pi(y|x)}{\hat{\pi}_0(y|x)}r(x,y)|x,y\Big)\\
  & = \mV_r\Big(\wxb\frac{\pi(y|x)}{\pi_0(y|x)}\frac{\pi_0(y|x)}{\hat{\pi}_0(y|x)}r|x,y\Big)\\
  & = c^2(\wxb)^2(1-\zeta)^2\mV_r(r|x,y)\\
  & = c^2(\wxb)^2(1-\zeta)^2\sigma^2_r
\end{split}
\end{equation}
where the first equality follows from the fact that conditioning on $(x,y)$, $\sum_{\bar{y}\in\mathcal{Y}} \pi(\bar{y}|x) \: \wxa \: \alpha_{x\bar{y}} +\pi(y|x) \: \wxg \: \gamma_{xy}$ is just a constant, and we use the formula $\mV(a+X)=\mV(X)$ for any constant $a$, random variable $X$.

Then for the term $\mE_x\Big[\mE_{y\sim \pi_0}[\mV_r(\rp|x,y)|x]\Big]$, we have
\begin{equation}
\begin{split}
  \mE_x\Big[\mE_{y\sim \pi_0}[\mV_r(\rp|x,y)|x]\Big] & = \mE_x\mE_{y\sim \pi_0}\Big[c^2(\wxb)^2(1-\zeta)^2\sigma^2_r\Big]\\
  & = \mE_x\mE_{y\sim \pi}\Big[c(\wxb)^2(1-\zeta)^2\sigma^2_r\Big]
\end{split}
\end{equation}

Similarly, for the third term, we will calculate $\mE_r[\rp|x,y]$ first.
\begin{equation}
\begin{split}
  \mE_r[\rp|x,y] & = \mE_r\Big[\sum_{\newy \in \mY}\pi(\newy|x)\wxa\alpha_{i\newy}+\pi(y|x)\wxb\beta_{xy}+\pi(y|x)\wxg\gamma_{xy}|x,y\Big]\\
  & =\sum_{\newy\in\mathcal{Y}}\pi(\newy|x)\wxa\hat{\delta}(x,\newy)+ \mE_r\Big[\wxb\frac{\pi(y|x)}{\pi_0(y|x)}\frac{\pi_0(y|x)}{\hat{\pi}_0(y|x)}r|x,y\Big]+\wxg\frac{\pi(y|x)}{\pi_0(y|x)}\frac{\pi_0(y|x)}{\hat{\pi}_0(y|x)}(\delta+\Delta)\\
  & = \sum_{\newy\in\mathcal{Y}}\pi(\newy|x)\wxa\hat{\delta}(x,\newy)+ \wxb\frac{\pi(y|x)}{\pi_0(y|x)}\frac{\pi_0(y|x)}{\hat{\pi}_0(y|x)}\delta+\wxg\frac{\pi(y|x)}{\pi_0(y|x)}\frac{\pi_0(y|x)}{\hat{\pi}_0(y|x)}(\delta+\Delta)\\
  & = \sum_{\newy\in\mathcal{Y}}\pi(\newy|x)\wxa\hat{\delta}(x,\newy)+ \wxb c(1-\zeta)\delta+\wxg c(1-\zeta)(\delta+\Delta)
\end{split}
\end{equation}
For the term $\mV_{y\sim\pi_0}\Big(\mE_r[\rp|x,y]|x\Big)$, since the first term $ \sum_{\newy\in\mathcal{Y}}\pi(\newy|x)\wxa\hat{\delta}(x,\newy)$ is independent of $y$, then we have
\begin{equation}
    \begin{split}
        \mV_{y\sim\pi_0}\Big(\mE_r[\rp|x,y]|x\Big) & = \mV_{y\sim\pi_0}\Big(\sum_{\newy\in\mathcal{Y}}\pi(\newy|x_i)\wxa\hat{\delta}(x_i,\newy)+ \wxb c(1-\zeta)\delta+\wxg c(1-\zeta)(\delta+\Delta)|x\Big)\\
        &=\mV_{y\sim\pi_0}\Big( \wxb c(1-\zeta)\delta+\wxg c(1-\zeta)(\delta+\Delta)|x\Big)
    \end{split}
\end{equation}
Then taking the outer expectation over $x$, we have:
\begin{equation}
    \mE_x\Big[\mV_{y\sim\pi_0}(\mE_r[\rp|x,y]|x)\Big] = \mE_x\Big[\mV_{y\sim\pi_0}( \wxb c(1-\zeta)\delta+\wxg c(1-\zeta)(\delta+\Delta)|x)\Big]
\end{equation}
Summing all the three terms together, and using the formula $\mV(\Rw(\pi))=\frac{1}{n}\mV(\Rwi(\pi))$ for $i.i.d$ $\Rwi$, we have:
\begin{equation}
\begin{split}
    \mV(\Rw(\pi)) =  & \frac{1}{n}\Big\{ \mathbb{V}_x\Big(\mathbb{E}_{\pi}[\wva\Delta - \wvb\zeta\delta + \wvg(\Delta-\zeta(\delta+\Delta))+(\wva+\wvb+\wvg)\delta]\Big)\\
    &+\mathbb{E}_{x}\mathbb{E}_{\pi}\Big[(\wvb)^2c(1-\zeta)^2\sigma^2_r\Big]+\mathbb{E}_{x}\Big[\mathbb{V}_{\pi_0}(\wvb c(1-\zeta)\delta+\wvg c(1-\zeta)(\delta+\Delta))\Big]\Big\}\\
    \end{split}
\end{equation}
\end{proof}

\subsection{Proof of Theorem 3}
\label{Appendix:prooftheo3}
\textbf{Theorem 3} (Bias of CAB). \textit{For contexts $x_1, x_2, \cdots, x_n$ drawn i.i.d from some distribution $P(\mX)$ and for actions $y_i\sim\pi_0(\mY|x_i)$, under Condition 1 the bias of $\Rcab(\pi)$ is}
\begin{equation}
\mathbb{E}_{x}\mathbb{E}_{\pi}[-\delta\zeta\indicator{\{\hat{c}\leq M\}} + \{\Delta(1-\frac{M}{c(1-\zeta)})-\frac{M}{c(1-\zeta)}\delta\zeta\} \indicator{\{\hat{c}>M\}}]
\end{equation}

\begin{proof}
Note CAB falls into the class of counterfactual estimator with the weighting functions $\wa = 1- \min\!\left\{M\frac{ \hat{\pi}_0(\newy|x_i)}{\pi(\newy|x_i)},1\right\}, \wb =\min\!\left\{M\frac{ \hat{\pi}_0(y_i|x_i)}{\pi(y_i|x_i)},1\right\}, \wg=0$.

Using Theorem~\ref{theo:biasunified}, the bias for $\Rcab(\pi)$ is:
\begin{equation}
\begin{split}
Bias(\Rcab(\pi)) &=\mE_x\mE_{y\sim\pi}\Big[\wxya\Delta - \wxb\zeta\delta + \wxg(\Delta-\zeta(\delta+\Delta))+(\wxa+\wxb+\wxg)\delta - \delta\Big]\\
&=\mathbb{E}_{x}\mathbb{E}_{y\sim\pi}\Big[(1- \min\{\frac{M}{\hat{c}(x,y)},1\})\Delta-\min\{\frac{M}{\hat{c}(x,y)},1\}\zeta\delta\Big]\\
&=\mathbb{E}_{x}\mathbb{E}_{y\sim\pi}\Big[-\zeta\delta\mCs+\{(1-\frac{M}{\hat{c}(x,y)})\Delta-\frac{M}{\hat{c}(x,y)}\zeta\delta\}\mCg\Big]\\
&=\mathbb{E}_{x}\mathbb{E}_{y\sim\pi}\Big[-\zeta\delta\mCs+\{(1-\frac{M}{c(1-\zeta)})\Delta-\frac{M}{c(1-\zeta)}\zeta\delta\}\mCg\Big]\\
\end{split}
\end{equation}
while the last equality follows from the fact that $\hat{c}(x,y):=\frac{\pi(y|x)}{\hat{\pi}_0(y|x)}=\frac{\pi(y|x)}{\pi_0(y|x)}\frac{\pi_0(y|x)}{\hat{\pi}_0(y|x)}=c(x,y)(1-\zeta(x,y))$
\end{proof}

\subsection{Proof of Theorem 4}
\label{Appendix:prooftheo4}
\textbf{Theorem 4} (Variance of CAB). \textit{Under the same conditions as in Theorem~\ref{theo:biasblend}, the variance of $\Rcab(\pi)$}
\begin{equation}
    \begin{split}
 \mV(\Rcab(\pi))&= \frac{1}{n}\Big\{ \mathbb{V}_x(\mathbb{E}_{\pi}[\delta -\delta\zeta\indicator{\{\hat{c}\leq M\}}+(\Delta(1-\frac{M}{c(1-\zeta)})-\frac{M}{c(1-\zeta)}\delta\zeta) \indicator{\{\hat{c}>M\}}])\\
        &+\mathbb{E}_{x}\mathbb{E}_{\pi}[c(1-\zeta)^2\sigma^2_r\indicator{\{\hat{c}\leq M\}} + \frac{M^2}{c}\sigma^2_r\indicator{\{\hat{c}> M\}}]+\mathbb{E}_{x}[\mathbb{V}_{\pi_0}(c(1-\zeta)\delta\indicator{\{\hat{c}\leq M\}}+M\delta\indicator{\{\hat{c}> M\}})]\Big\}\\
    \end{split}
\end{equation}

\begin{proof}
The result follows by plugging in the weighting function for CAB with $\wa = 1- \min\!\left\{M\frac{ \hat{\pi}_0(\newy|x_i)}{\pi(\newy|x_i)},1\right\}, \wb =\min\!\left\{M\frac{ \hat{\pi}_0(y_i|x_i)}{\pi(y_i|x_i)},1\right\}, \wg=0$ in Theorem~\ref{theo:varianceunified}.

For the term $\mathbb{V}_x\Big(\mathbb{E}_{\pi}[\wva\Delta - \wvb\zeta\delta + \wvg(\Delta-\zeta(\delta+\Delta))+(\wva+\wvb+\wvg)\delta]\Big)$, using the result from Theorem~\ref{theo:biasblend}, we have:
\begin{equation}
\begin{split}
    &  \mathbb{V}_x\Big(\mathbb{E}_{\pi}[\wva\Delta - \wvb\zeta\delta + \wvg(\Delta-\zeta(\delta+\Delta))+(\wva+\wvb+\wvg)\delta]\Big)\\
    &=\mathbb{V}_x\Big(\mathbb{E}_{\pi}[\delta -\delta\zeta\indicator{\{\hat{c}\leq M\}}+(\Delta(1-\frac{M}{c(1-\zeta)})-\frac{M}{c(1-\zeta)}\delta\zeta) \indicator{\{\hat{c}>M\}}]\Big)
\end{split}
\end{equation}

For the term $\mathbb{E}_{x}\mathbb{E}_{\pi}\Big[(\wvb)^2c(1-\zeta)^2\sigma^2_r\Big]$, we have
\begin{equation}
    \begin{split}
        \mathbb{E}_{x}\mathbb{E}_{\pi}\Big[(\wvb)^2c(1-\zeta)^2\sigma^2_r\Big] & = \mathbb{E}_{x}\mathbb{E}_{\pi}\Big[\min\{(\frac{M}{\hat{c}(x,y)})^2,1\}c(1-\zeta)^2\sigma^2_r\Big]\\
        &=\mE_x\mE_{\pi}\Big[c(1-\zeta)^2\sigma^2_r\mCs + \frac{M^2}{\hat{c}^2(x,y)}c(1-\zeta)^2\sigma^2_r\mCg\Big]\\
        &=\mE_x\mE_{\pi}\Big[c(1-\zeta)^2\sigma^2_r\mCs + \frac{M^2}{c^2(1-\zeta)^2}c(1-\zeta)^2\sigma^2_r\mCg\Big]\\
        &=\mE_x\mE_{\pi}\Big[c(1-\zeta)^2\sigma^2_r\mCs + \frac{M^2}{c}\sigma^2_r\mCg\Big]
    \end{split}
\end{equation}

For the last term $\mathbb{E}_{x}\Big[\mathbb{V}_{\pi_0}(\wvb c(1-\zeta)\delta+\wvg c(1-\zeta)(\delta+\Delta))\Big]$, then
\begin{equation}
    \begin{split}
        \mathbb{E}_{x}\Big[\mathbb{V}_{\pi_0}(\wvb c(1-\zeta)\delta)\Big] &= \mathbb{E}_{x}\Big[\mathbb{V}_{\pi_0}(\min\{\frac{M}{\hat{c}(x,y)},1\}c(1-\zeta)\delta)\Big]\\
        &=\mE_x\Big[\mV_{\pi_0}(c(1-\zeta)\delta\mCs + \frac{M}{\hat{c}(x,y)}c(1-\zeta)\delta\mCg)\Big]\\
        &=\mE_x\Big[\mV_{\pi_0}(c(1-\zeta)\delta\mCs + M\delta\mCg)\Big]\\
    \end{split}
\end{equation}
Combining all, we have
\begin{equation}
    \begin{split}
 \mV(\Rcab(\pi))&= \frac{1}{n}\Big\{ \mathbb{V}_x\Big(\mathbb{E}_{\pi}[\delta -\delta\zeta\indicator{\{\hat{c}\leq M\}}+(\Delta(1-\frac{M}{c(1-\zeta)})-\frac{M}{c(1-\zeta)}\delta\zeta) \indicator{\{\hat{c}>M\}}]\Big)\\
        &+\mathbb{E}_{x}\mathbb{E}_{\pi}\Big[c(1-\zeta)^2\sigma^2_r\indicator{\{\hat{c}\leq M\}} + \frac{M^2}{c}\sigma^2_r\indicator{\{\hat{c}> M\}}\Big]+\mathbb{E}_{x}\Big[\mathbb{V}_{\pi_0}(c(1-\zeta)\delta\indicator{\{\hat{c}\leq M\}}+M\delta\indicator{\{\hat{c}> M\}})\Big]\Big\}\\
    \end{split}
\end{equation}
\end{proof}

\subsection{Proof of Bias and Variance of CAB-DR}
\label{Appendix:prooftheo5}
\textbf{Theorem 5} (Bias of CAB-DR). \textit{For contexts $x_1,x_2,\cdots,x_n$ drawn i.i.d from some distribution $P(\mX)$ and for actions $y_i\sim\pi_0(\mY|x_i)$, under Condition~\ref{as:commonsupport} the bias of $\Rcabdr(\pi)$ is}
\begin{equation}
\begin{split}
\mE_x\mE_{\pi}\Big[\zeta\Delta\indicator{\{\hat{c}\leq M\}}+\Delta(1-\frac{M}{c}) \indicator{\{\hat{c}> M\}}\Big]\\
\end{split}
\end{equation}

\begin{proof}
CAB-DR is also an instance in the \classname\ with the weighting function: $\wa = 1, \wb =\min\!\left\{M\frac{ \hat{\pi}_0(y_i|x_i)}{\pi(y_i|x_i)},1\right\}, \wg=- \min\!\left\{M\frac{ \hat{\pi}_0(y_i|x_i)}{\pi(y_i|x_i)},1\right\}$. Using Theorem~\ref{theo:biasunified}, the bias for CAB-DR is:
\begin{equation}
    \begin{split}
        Bias(\hat{R}_{CABDR}(\pi)) &= \mE_x\mE_{\pi}\Big[\wva \Delta - \wvb\zeta\delta + \wvg(\Delta-\zeta(\delta+\Delta))+(\wva+\wvb+\wvg)\delta - \delta\Big]\\
        &=\mE_x\mE_{\pi}\Big[\Delta - \min\{\frac{M}{\hat{c}(x,y)},1\}\zeta\delta -\min\{\frac{M}{\hat{c}(x,y)},1\}(\Delta-\zeta(\delta+\Delta))\Big]\\
        &=\mE_x\mE_{\pi}\Big[\Delta  -\min\{\frac{M}{\hat{c}(x,y)},1\}(\Delta-\zeta\Delta)\Big]\\
        &=\mE_x\mE_{\pi}\Big[\zeta\Delta\mCs + \{\Delta[1-\frac{M}{\hat{c}(x,y)}(1-\zeta)]\}\mCg\Big]\\
        &=\mE_x\mE_{\pi}\Big[\zeta\Delta\indicator{\{\hat{c}\leq M\}}+\Delta(1-\frac{M}{c}) \indicator{\{\hat{c}> M\}}\Big]
    \end{split}
\end{equation}
\end{proof}

\textbf{Theorem 6} (Variance of CAB-DR). \textit{Under the same conditions as in Theorem~\ref{theo:biasblend}, the variance of $\Rcabdr(\pi)$}
\begin{equation}
\begin{split}
\mV(\Rcabdr(\pi)) &= \frac{1}{n}\Big\{ \mV_x\Big(\mE_{\pi}[\delta+\zeta\Delta\indicator{\{\hat{c}\leq M\}}+\Delta(1-\frac{M}{c}) \indicator{\{\hat{c}> M\}}]\Big)\\ &+\mathbb{E}_{x}\mathbb{E}_{\pi}\Big[c(1-\zeta)^2\sigma^2_r\indicator{\{\hat{c}\leq M\}} + \frac{M^2}{c}\sigma^2_r\indicator{\{\hat{c}> M\}}\Big]\\
&+\mathbb{E}_{x}\Big[\mathbb{V}_{\pi_0}(c(1-\zeta)(-\Delta)\indicator{\{\hat{c}\leq M\}}-M\Delta\indicator{\{\hat{c}> M\}})\Big]\Big\}\\
\end{split}
\end{equation}
\begin{proof}
The proof follows by using Theorem~\ref{theo:varianceunified} with the weights $\wa = 1, \wb =\min\!\left\{M\frac{ \hat{\pi}_0(y_i|x_i)}{\pi(y_i|x_i)},1\right\}, \wg=- \min\!\left\{M\frac{ \hat{\pi}_0(y_i|x_i)}{\pi(y_i|x_i)},1\right\}$.

For the first term, following directly from Appendix~\ref{Appendix:prooftheo5}, it is easy to see
\begin{equation}
   \mathbb{V}_x\Big(\mathbb{E}_{\pi}[\wva\Delta - \wvb\zeta\delta + \wvg(\Delta-\zeta(\delta+\Delta))+(\wva+\wvb+\wvg)\delta]\Big)= \mV_x\Big(\mE_{\pi}[\delta+\zeta\Delta\indicator{\{\hat{c}\leq M\}}+\Delta(1-\frac{M}{c}) \indicator{\{\hat{c}> M\}}]\Big)
\end{equation}

For the second term $\mathbb{E}_{x}\mathbb{E}_{\pi}[(\wvb)^2c(1-\zeta)^2\sigma^2_r]$, since CAB and CAB-DR has the same weighting function $\wvb$, this term is exactly the same for CAB and CAB-DR.

For the third term $\mathbb{E}_{x}\Big[\mathbb{V}_{\pi_0}(\wvb c(1-\zeta)\delta+\wvg c(1-\zeta)(\delta+\Delta))\Big]$, we have
\begin{equation}
\begin{split}
    \mathbb{E}_{x}\Big[\mathbb{V}_{\pi_0}(\wvb c(1-\zeta)\delta+\wvg c(1-\zeta)(\delta+\Delta))\Big]
    &=\mathbb{E}_{x}\Big[\mathbb{V}_{\pi_0}(\min\{\frac{M}{\hat{c}(x,y)},1\}c(1-\zeta)\delta-\min\{\frac{M}{\hat{c}(x,y)},1\}c(1-\zeta)(\delta+\Delta))\Big]\\
    &=\mathbb{E}_{x}\Big[\mathbb{V}_{\pi_0}(-\min\{\frac{M}{\hat{c}(x,y)},1\}c(1-\zeta)\Delta)\Big]\\
    &=\mathbb{E}_{x}\Big[\mathbb{V}_{\pi_0}(c(1-\zeta)(-\Delta)\mCs - M\Delta\mCg)\Big]
\end{split}
\end{equation}
Combining all the three terms will give us the variance for $\Rcabdr(\pi)$.
\end{proof}

\section{Experiment Details}
\label{appendix: setup}
In this section, we provide experiment details for both the BLBF and LTR settings.
\subsection{BLBF}
In the BLBF experiment, specifically, given a supervised dataset $\{(x_i, y^{*}_i)\}^n_{i=1}$, where $x$ is $i.i.d$ drawn from a certain fixed distribution $P(\mathcal{X})$ and $y^{*}\in\{1,2,\cdots,k\}$ denotes the true class label. For a particular logging policy $\pi_0$, the logged bandit data is simulated by sampling $y_i \sim \pi_0(\mathcal{Y}|x_i)$ and a deterministic loss $r(x_i,y_i)$ is revealed. In our experiments, the loss is defined as $r(x_i,y_i) = \indicator{\{y_i \neq y^{*}_i\}}-1$. The resulting logged contextual bandit data $\mathcal{S}=\{x_i, y_i, r(x_i, y_i), \pi_0(y_i|x_i)\}$ is then used to evaluate the performance of different estimators.

For evaluation, we split each dataset equally into train and test sets. For the train set, we use $10\%$ of the full-information data to train the logger $\pi_0$ and loss predictor $\pi_r(x)$, with loss estimates defined by $\hat{\delta}(x_i,y) = \indicator{\{\pi_r(x_i)\neq y\}}-1$. The policy $\pi$ we want to evaluate is a multiclass logistic regression trained on the whole train set. Finally, we use the full-information test set to generate the contextual bandit datasets $\mathcal{S}$ for off-policy evaluation of sizes $n=200, 500, 2000$. We evaluate the policy $\pi$ with different estimators on the logged bandit feedback of different sizes and treat the performance on the full-information test set as ground truth $R(\pi)$. The performance is measured by MSE. We repeat each experiment 500 times and calculate the bias, variance and MSE.

For learning, we first split the original dataset into training (48\%), validation (32\%) and test sets (20\%). Following \cite{Swaminathan/Joachims/15c}, the policy we want to learn lies on the space $\mathcal{F}:=\{\pi_w:w\in\mathbb{R}^p\}$ with $\pi_w$ as the stochastic linear rules defined by:
\begin{equation}
    \pi_w(y|x) = \frac{\exp(w^T\phi(x,y))}{\mathbb{Z}(x)}
\end{equation}
Here, $\phi(x,y)$ denotes the joint feature map between context $x$ and action $y$, and $\mathds{Z}(x)$ is a normalization factor.
The training objective is defined by $ \pi^{est} = \argmin_{\pi_w \in \mathcal{F}}\hat{R}^{est}(\pi_w) + \lambda||w||_2$, where $\lambda$ is selected through the lowest $\hat{R}^{IPS}(\pi)$ on the validation set. To avoid local minimum, the objective is optimized via L-BFGS using scikit-learn with 10 random starts. The performance of the learned policy $\pi^{est}$ is measured via expected error on the test set, defined as: $\frac{1}{n_{test}}\sum_{i=1}^{n_{test}}\mE_{y_i\sim \pi^{est}(\mathcal{Y}|x_i)}[\indicator{\{y_i \neq y^{*}_i\}}]$. Similar to evaluation, we use 20\% of the training data to train the multiclass logistics regression as logging policy with default hyperparameter. While for the estimated loss $\hat{\delta}(x,y)$, we train the logistic regression using 10\% training data with tuned hyperparameter selected from the validation set. All the result is averaged over 10 runs with $n=5000$.

\subsection{LTR} \label{sec:genpropsvm}
In the LTR experiment, we use 10\% of the training set for learning a DM, which reflects that we typically have a small amount of manual relevance judgements. The DM is a binary Gradient Boosted Decision Tree Classifier calibrated by Sigmoid Calibration~\cite{platt1999probabilistic}. We use $\lambda(rank)=rank$ as the performance metric which can be interpreted as the average rank of the relevant results. The examination probability (propensity) that we use is $p(x,\pi_0(x),d) = \frac{1}{rank(d|\pi_0(x))}$.
For the evaluation experiments, to get a ranking policy for evaluation, we train a ranking SVM \cite{joachims2002optimizing} on the remaining 90\% training data. As input to the estimators, different amounts of click data are generated from the test set. For each experiment, we generate the log data 100 times and report the bias, variance, and MSE with respect to the estimated ground truth from the full-information test set.

For the learning experiments, we derived a concrete learning algorithm based on propensity SVM-Rank that conducts learning from biased user feedback using different estimators. The SVM-style algorithm~\cite{joachims2002optimizing,joachims2006training,joachims2017unbiased} optimizes an upper bound on different estimators and details are in Appendix~\ref{Appendix:svm-rank}. We compare the performance of different estimators using different amounts of simulated user feedback with the proposed learning algorithm.
As input to the propensity SVM-Rank, different amount of click data is simulated from the $90\%$ training data. Specifically, we present the performance using 1 sweep of the data in Table~\ref{datasets-table}. We grid search $C$ for propensity SVM-Rank and $M$ for different estimators and conduct hyperparameter selection with 90 percentile cIPS on user feedback data simulated from the validation set for 5 sweeps. All the experiments are run for 5 times and the average is presented.

\subsection{Generalized Propensity SVM-Rank}
\label{Appendix:svm-rank}
We now derive a concrete learning algorithm that conducts learning from biased user feedback using different estimators from the \classname. It is based on SVM-Rank~\cite{joachims2002optimizing,joachims2006training,joachims2017unbiased} but we expect other learning to rank methods can also be adapted to the estimators.

The generalized propensity SVM-Rank learns a linear scoring function $f(x,d)=w\cdot\phi(x,d)$ with $\phi(x,d)$ describing how context $x$ and document $d$ interact. It optimizes the following objective
\begin{equation}
    \begin{split}
        &\hat{w} = argmin_{w,\xi}\frac{1}{2}w\cdot w+
        \frac{C}{n}\sum_{i}\sum_{j}\bigg[w_{ij}^{\alpha}\alpha_{ij}+o_{ij}w_{ij}^{\beta} \beta_{ij}\bigg]\sum_{k\neq j}\xi_{ijk}\\
        &s.t.\quad \forall i, j, k\neq j\quad w\cdot [\phi(x_i,d_{ij})-\phi(x_i,d_{ik})]> 1-\xi_{ijk},\\ 
        & \hspace{0.83cm} \forall i,j,k\neq j \quad \xi_{ijk}\geq 0
    \end{split}
\end{equation}
where $w$ is the parameter of the generalized propensity SVM-Rank and $C$ is a regularization parameter. The training objective optimizes an upper bound on the estimator with average rank of positive examples metric($\lambda(rank) = rank$) since
\begin{align}
\begin{split}
    &\sum_i\sum_j\bigg[w_{ij}^{\alpha}\alpha_{ij}+o_{ij}w_{ij}^{\beta} \beta_{ij}\bigg](rank(d_{ij}|x_i,\pi(x_i))-1)\nonumber \\
    =&\sum_i\sum_j\bigg[w_{ij}^{\alpha}\alpha_{ij}+o_{ij}w_{ij}^{\beta} \beta_{ij}\bigg]\cdot \sum_{k\neq j}\indicator\{w\cdot[\phi(x_i,d_{ik})-\phi(x_i,d_{ij})]>0\}\nonumber \\
    \leq&\sum_i\sum_j\bigg[w_{ij}^{\alpha}\alpha_{ij}+o_{ij}w_{ij}^{\beta} \beta_{ij}\bigg]\cdot \sum_{k\neq j}\max(1-w\cdot[\phi(x_i,d_{ij})-\phi(x_i,d_{ik})],0)\\
    \leq&\sum_i\sum_j\bigg[w_{ij}^{\alpha}\alpha_{ij}+o_{ij}w_{ij}^{\beta} \beta_{ij}\bigg]\cdot\sum_{k\neq j}\xi_{ijk}
\end{split}
\end{align}
\end{document}